\documentclass{article}

\PassOptionsToPackage{numbers, compress}{natbib}

\usepackage[preprint]{neurips_2026}

\usepackage[utf8]{inputenc} 
\usepackage[T1]{fontenc}    
\usepackage{url}            
\usepackage{booktabs}       
\usepackage[pagebackref, breaklinks=true, colorlinks, citecolor=bluecite, bookmarks=false]{hyperref}
\usepackage{amsfonts}       
\usepackage{nicefrac}       
\usepackage{microtype}      
\usepackage{xcolor}
\definecolor{bluecite}{HTML}{0071BC}

\usepackage{graphicx}
\usepackage{subcaption}
\usepackage{multirow}

\usepackage{amsmath}

\newtheorem{lemma}{Lemma}[section]
\newtheorem{proof}{Proof}[section]
\usepackage{paralist}

\newtheorem{proposition}{Proposition}
\newtheorem{assumption}{Assumption}

\usepackage[capitalize]{cleveref}
\crefname{section}{Sec.}{Secs.}
\Crefname{section}{Section}{Sections}
\Crefname{table}{Table}{Tables}
\crefname{table}{Tab.}{Tabs.}

\usepackage{microtype}
\usepackage{booktabs} %

\definecolor{darkgreen}{rgb}{0.0, 0.5, 0.0} 
\usepackage{algorithm}
\usepackage{algpseudocode}

\newcommand{\comt}[1]{#1}

\renewcommand{\comt}[1]{}

\usepackage[dvipsnames,table]{xcolor}

\usepackage{tabularx}
\usepackage{multicol}
\definecolor{myblue}{RGB}{235,235,250}

\usepackage{xspace}
\usepackage{pifont}
\usepackage{latexsym}
\usepackage{inconsolata}
\usepackage{arydshln}
\usepackage{enumitem}
\usepackage{wrapfig}
\usepackage{array}
\usepackage{epigraph}
\usepackage{amssymb}
\usepackage[most,skins]{tcolorbox}
\definecolor{lightpink}{RGB}{204, 231, 207} 
\definecolor{lightblue}{RGB}{210, 220, 250} 
\definecolor{lightgray}{RGB}{237, 237, 237} 

\definecolor{superlightred}{rgb}{0.99, 0.92, 0.92}
\definecolor{darkgreen}{RGB}{50,100,0}
\definecolor{darkred}{RGB}{200, 0, 0}
\newcommand{\cmark}{\textcolor{darkgreen}{\ding{51}}} %
\newcommand{\xmark}{\textcolor{darkred}{\ding{55}}} %

\usepackage{makecell}
\usepackage{colortbl}

\usepackage{bbm}

\usepackage{tcolorbox}

\definecolor{myboxcolor}{RGB}{255,255,255} 
\definecolor{myframe}{RGB}{128,128,128} 
\newtcolorbox{mybody}{
  colback=myboxcolor,
  colframe=myframe,
  boxrule=1pt, 
  left=1pt,
  right=1pt,
  top=1pt,
  bottom=1pt,
}

\definecolor{my_green}{RGB}{51,102,0}
\definecolor{my_yellow}{RGB}{255,165,0}
\definecolor{my_red}{RGB}{204, 0, 0}

\definecolor{backred}{RGB}{255, 190, 190}
\definecolor{backblue}{RGB}{210, 230, 250}

\definecolor{shadecolor}{RGB}{237,237,237}

\usepackage[toc]{appendix}
\usepackage{etoc}
\usepackage{minitoc}
\etocsettocstyle{\section*{Contents of Technical Appendices}}{}

\title{Consistency as Inductive Bias: Learning Cross-View Invariance for Robust  Multimodal Reasoning}

\author{%
  Xin Zou$^{1,2}$, Haolin Deng$^{1,2}$, Yibo Yan$^{1,2}$, Shuliang Liu$^{1,2}$, Kening Zheng$^{4}$, \\\textbf{Zhiwei Jin$^{3}$, Chen Chen$^{3}$, Haonan Lu$^{3}$, Xuming Hu$^{1,2}$\thanks{Corresponding author. <dylan.zoux@gmail.com>}}  \\ [2.5pt]
  $^{1}$HKUST (GZ),
  $^{2}$HKUST,
  $^{3}$OPPO AI Center,
  $^{4}$UIC
}

\begin{document}

\maketitle

\etocdepthtag.toc{mtchapter}
\etocsettagdepth{mtchapter}{subsection}
\etocsettagdepth{mtappendix}{none}

\begin{abstract}
Inductive biases steer learning toward generalizable solutions by encoding task structure. In this work, we identify a crucial missing bias in MLLMs: cross-view consistency, \textit{i.e.}, semantically invariant views of the same instance should lead to the same answer. Standard reinforcement learning with verifiable rewards (RLVR) objectives do not impose this constraint, but instead assign pointwise rewards to each visual input. Even with data augmentation (DA), transformed views are typically rewarded independently, providing little signal once within-view rewards saturate. We propose \textbf{ConsistRoll}, a simple but effective method that injects cross-view consistency into RLVR training by reusing the group-sampling mechanism of GRPO. Specifically, ConsistRoll places original and semantically invariant transformed views in the same generation group, and assigns a joint reward only when paired completions are both correct and consistent. In this way, ConsistRoll turns consistency into an online credit-assignment signal, \textbf{without extra generation overhead and annotations}. Theoretically, we show that cross-view consistency is a valid inductive bias, and ConsistRoll introduces a cross-view correction term absent from DA, penalizing view dependence and alleviating advantage collapse. Comprehensive benchmarks across math, general-purpose, hallucination domains confirm that ConsistRoll achieves robust improvements in multimodal reasoning.
\end{abstract}

\section{Introduction}
\label{sec:intro}

\begin{wrapfigure}{r}{0.55\textwidth}
    \centering \vspace{-2.8em}
    \includegraphics[width=1\linewidth, trim={25pt 10pt 25pt 15pt}, clip]{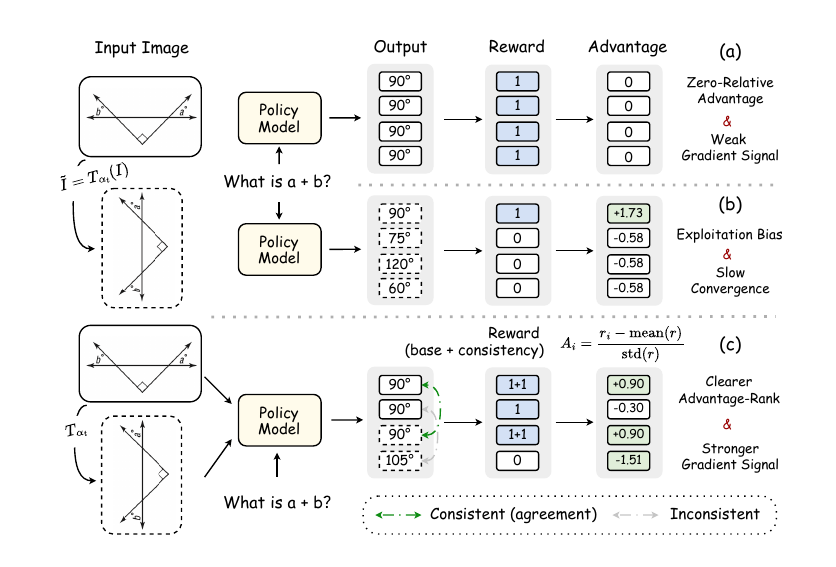} \vspace{-1.75em}
    \caption{(a) Naive GRPO and (b) it with DA strategy rewards views independently, while ConsistRoll turns cross-view agreement into on-policy credit assignment.}
    \label{fig:method-overview}\vspace{-1.6em}
\end{wrapfigure}
Multimodal large language models (MLLMs) have made rapid progress in visual reasoning~\cite{bai2025qwen3,li2024llavaonevision,yin2024survey}, but their reliability remains limited by hallucination and spurious visual grounding~\cite{li2023pope,wang2023amber,guan2024hallusionbench}. Recently, Group Relative Policy Optimization (GRPO)~\cite{shao2024deepseekmath} has become an attractive post-training paradigm for this setting, because many visual reasoning tasks can be evaluated by verifiable answers rather than dense rationale supervision~\cite{guo2025deepseekr1,liu2025visualreasoning}. However, standard GRPO treats each image-question pair as an independent reward unit and normalizes rewards within a single visual view, leaving cross-view consistency outside the learning objective. When the same question-relevant semantics are presented under different semantically invariant views, a reliable model should not only be correct on each view separately, but also produce stable answers across these views. This issue is different from the known zero-advantage case in GRPO, where equal rewards within a group eliminate relative advantages~\cite{zhang2025edge}. Here, as shown in Figure~\ref{fig:method-overview}, the missing signal is cross-view contrast: GRPO can provide valid within-view advantages while still failing to compare original and transformed views. DA~\cite{li2025eagle} improves coverage, but the two views are still rewarded as separate references under a pointwise objective, leaving consistency unoptimized.

Classical vision architectures often encode invariances through their design, \textit{e.g.}, translation equivariance in CNNs~\cite{lecun1989} or permutation invariance in set and graph models~\cite{zaheer2017deepsets}. Modern MLLMs, however, are usually assembled from a visual encoder, a projector, and an autoregressive language model, where visual tokens are serialized, and positional encodings and attention patterns do not naturally guarantee rotation, scale, or viewpoint invariance. Moreover, these models are already pretrained at a large scale, making architectural redesign costly and incompatible with standard post-training pipelines. Therefore, the inductive bias of consistency should be injected into the policy update rather than into the architecture. Recent MLLM reinforcement learning methods~\citep{liu2025visualreasoning,huang2025visionr1,tan2025reasonrft,zhang2025r1vl,shen2025vlmr1} extend GRPO-style RLVR to visual domains. Some methods, such as Share-GRPO and NoisyRollout~\citep{yao2025sharegrpo,liu2025noisyrollout}, perturb inputs to increase rollout diversity, but transformed views remain additional independent samples. Other perception-aware variants, such as PAPO and VPPO~\citep{wang2026papo,huang2026vppo}, use corrupted or masked visual inputs as auxiliary signals to measure visual dependence and shape the original on-policy update, but they still do not jointly optimize answer agreement between equivalent views. We argue for a fundamentally different perspective: a semantically invariant transformation defines a counterfactual view of the same reasoning problem. \textit{To enable on-policy cross-view invariance, rollouts under different views must not be rewarded in isolation, but jointly optimized via a shared reward scheme.}

\begin{figure}[t]
    \centering \vspace{-0.5em}
    \includegraphics[width=1\linewidth]{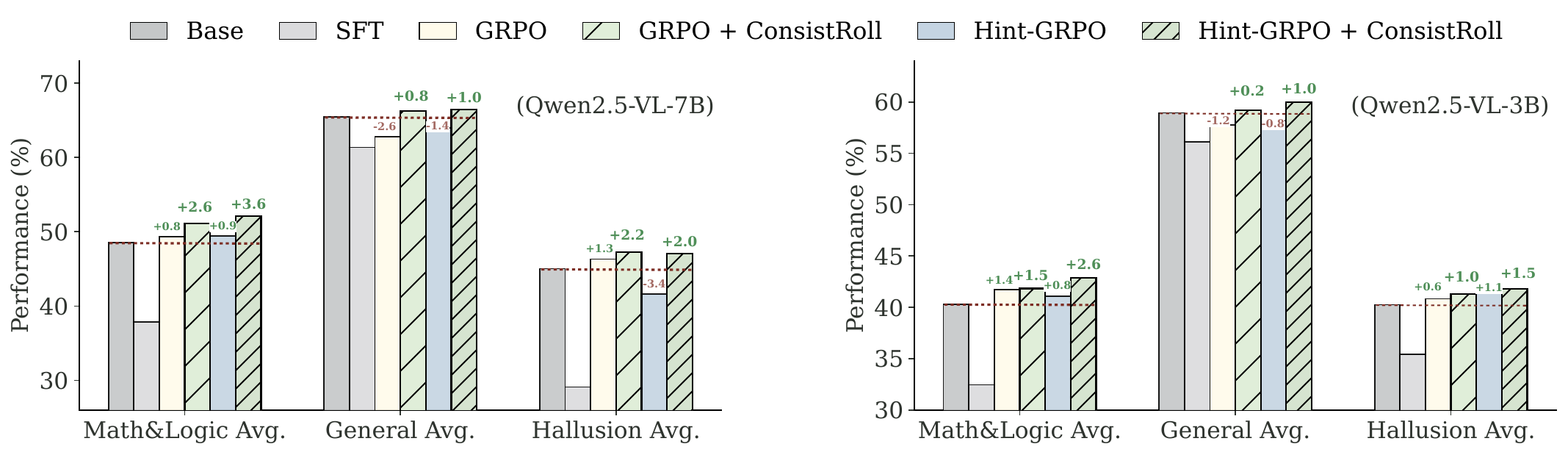}\vspace{-0.75em}
    \caption{
Performance on comprehensive benchmarks. ConsistRoll consistently improves GRPO and Hint-GRPO, showing that cross-view consistency is an effective inductive bias for RLVR training.}
    \label{fig2:performance} \vspace{-1.5em}
\end{figure}
To this end, we propose \textbf{ConsistRoll}, a simple yet effective method that injects cross-view consistency into RLVR training by repurposing the group-sampling mechanism of GRPO. Instead of treating augmented inputs as independent training samples, ConsistRoll places the completions generated from an original image and its transformed counterpart into the same rollout group. It then assigns a joint reward only when both responses are individually correct and consistent. Since the paired views share the same group statistics, their advantages become coupled, yielding the cross-view correction illustrated in Figure~\ref{fig:method-overview}. ConsistRoll activates the consistency branch only for transformations known to preserve the target answer, and otherwise falls back to ordinary pointwise reward optimization. Theoretically, we show that ConsistRoll introduces a cross-view correction term into the policy gradient, a term absent from naive DA, encouraging diverse exploration and robust visual understanding, while preserving the standard GRPO learning signal when no valid consistent pair is available.

Empirically, ConsistRoll delivers consistent improvements across model scales, RL recipes, and benchmark families. As summarized in Figure~\ref{fig2:performance}, adding ConsistRoll improves both GRPO and Hint-GRPO on Qwen2.5-VL-7B and Qwen2.5-VL-3B, spanning math-and-logic reasoning, general multimodal understanding, and hallucination-sensitive evaluation. These gains indicate that cross-view consistency is not merely an augmentation heuristic, but a transferable inductive bias for GRPO-style post-training. The effect is most pronounced when pointwise RLVR already solves many canonical examples but remains unstable under answer-preserving visual transformations. 

Overall, our contributions are summarized as follows:
\vspace{-0.5em}
\begin{itemize}[leftmargin=*]
  \item We identify a view-wise credit assignment gap in multimodal RLVR: standard pointwise rewards optimize per-view correctness but fail to encode consistency across semantically invariant views.
  \item We propose ConsistRoll that jointly rewards cross-view semantically-invariant pairs. Our theoretical analysis shows that cross-view consistency serves as a valid and learnable inductive bias for RLVR.
  \item We extensively validate ConsistRoll across 14 benchmarks, showing consistent improvements in mathematical, logical, and general multimodal reasoning, and stronger cross-view consistency.
\end{itemize}


\section{Related Work}
\label{sec:related}

\textbf{RLVR for multimodal reasoning.}
Reinforcement learning has become a central post-training paradigm for large language models and MLLMs, from RLHF with learned reward models~\citep{ziegler2019rlhf,ouyang2022instructgpt} to RLVR, where rule-based answer checkers provide scalable supervision. GRPO~\citep{shao2024deepseekmath} removes the critic by normalizing rewards within a sampled group, and has been widely adopted in recent reasoning models such as DeepSeek-R1~\citep{guo2025deepseekr1} and Kimi k1.5~\citep{team2025kimi}. This paradigm has quickly moved into the multimodal domain. Visual-RFT~\citep{liu2025visualreasoning}, Vision-R1~\citep{huang2025visionr1}, Reason-RFT~\citep{tan2025reasonrft}, R1-VL~\citep{zhang2025r1vl}, VLM-R1~\citep{shen2025vlmr1}, and Perception-R1~\citep{xiao2025perceptionr1} show that GRPO-style RL can improve visual perception, grounding, and multimodal reasoning. Recent variants further refine rollout construction and reward propagation, such as Share-GRPO~\citep{yao2025sharegrpo} and NoisyRollout~\citep{liu2025noisyrollout}. ConsistRoll is complementary to this line: rather than designing denser rewards or improving exploration alone, it changes the \emph{comparison unit} by placing semantically invariant views within a group, enabling efficient visual-dependent exploration.

\textbf{Augmentation, consistency, and invariance.}
Data augmentation is a standard mechanism for injecting invariances into vision systems, with classical and automated policies such as AutoAugment, RandAugment, and AugMix improving robustness by exposing models to label-preserving transformations~\citep{cubuk2019autoaugment,cubuk2020randaugment,hendrycks2020augmix}. Prediction consistency across perturbations has also been widely studied in semi-supervised and robust learning, including Mean Teacher, Virtual Adversarial Training, and UDA~\citep{tarvainen2017mean,miyato2018vat,xie2020uda}. These methods typically impose consistency through supervised or semi-supervised probability-matching losses. ConsistRoll instead operates in a language-generation RL setting, where answers are sampled, parsed, and scored by verifiable rewards. This distinction is crucial: consistency alone admits degenerate policies that repeatedly output the same wrong answer, while correctness alone reduces to pointwise RLVR. ConsistRoll uses a correctness-aware consistency predicate, turning agreement across equivalent views into a safe online credit-assignment signal.

\textbf{Inductive biases, hallucination, and spurious grounding.}
The broader motivation of ConsistRoll follows the principle that models generalize better when their inductive biases match task symmetries. Classical architectures encode such biases structurally, including translation equivariance in CNNs~\citep{lecun1989}, permutation invariance in set and graph models~\citep{zaheer2017deepsets}, and symmetry constraints in geometric deep learning~\citep{cohen2016group,weiler2019general,bronstein2021geometric}. Modern MLLMs, however, are typically adapted through post-training rather than architectural redesign. ConsistRoll therefore encodes invariance algorithmically through the RL objective. This is particularly relevant for hallucination and spurious visual grounding, which remain central reliability failures in MLLMs~\citep{li2023pope,wang2023amber,guan2024hallusionbench}. Related preference-optimization methods reduce hallucination through human, model-generated, or synthetic preference pairs~\citep{sun2023llavarlhf,yu2024rlhfv,wang2024mdpo,xie2024vdpo,fu2025chip}. In contrast, ConsistRoll targets tasks with verifiable answers and no preference annotations, penalizing view-specific shortcuts through paired, semantically-invariant visual transformations. We provide a more detailed discussion of related work in Appendix~\ref{apx:related}.
\section{ConsistRoll: A Free Inductive Bias for Cross-View Consistency}
\label{sec:method}
ConsistRoll is a lightweight modification of GRPO that injects cross-view consistency into reinforcement learning with verifiable rewards. 
The central idea is simple: if two visual views preserve the answer to the same question, they should not be optimized as unrelated training instances. 
Instead, their rollouts should be compared under a shared reward-normalization group, so that correctness and cross-view stability jointly determine the relative advantage. 
ConsistRoll realizes this idea by changing only the rollout grouping and reward construction, while leaving the model architecture, answer parser, and GRPO optimizer unchanged. Importantly, it keeps the total rollout budget the same as standard GRPO: for an invariant pair, the original view and the transformed view each sample $n/2$ completions, yielding $n$ rollouts in total. Figure~\ref{fig:consistroll-framework} illustrates the framework of our ConsistRoll.

\begin{figure}[h]
    \centering\vspace{-0.5em}
    \includegraphics[width=1\linewidth, trim={18pt 5pt 22pt 8pt}, clip]{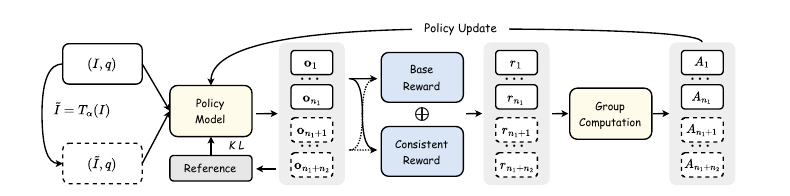}
\caption{ConsistRoll overview. The original and transformed views are sampled within a shared GRPO group. A paired consistency reward is assigned only when the two completions are both correct and consistent, turning cross-view invariance into consistency-aware credit-assignment signals.}\vspace{-0.5em}
\label{fig:consistroll-framework}
\end{figure}

\subsection{Preliminaries: Group Relative Policy Optimization (GRPO)}
\label{sec:grpo-prelim}

GRPO~\citep{shao2024deepseekmath} streamlines the PPO framework~\citep{schulman2017ppo} by eliminating the need for a separate value model. Instead, it estimates advantages by normalizing verifiable rewards within a sampled group. For a given input pair $(I, q)$ consisting of an image and text query, the reference policy $\pi_{\theta_{\mathrm{old}}}$ generates a group of $n$ response rollouts $\{o_1, \dots, o_n\}$. A rule-based parser evaluates each completion, assigning a pointwise outcome reward $r_i = r(I, q, o_i)$, e.g., $1$ for a correct format or answer, $0$ otherwise.
The relative advantage for each rollout is computed by normalizing these rewards: $\hat{A}_i = \frac{r_i - \mathrm{mean}(\mathbf{r})}{\mathrm{std}(\mathbf{r})}$,
and the policy $\pi_\theta$ is then updated via the clipped surrogate GRPO objective:
\begin{equation}
\label{eq:grpo-obj}
\mathcal{J}_{\mathrm{GRPO}}(\theta) = \mathbb{E}\left[ \frac{1}{n} \sum_{i=1}^{n} \min\left( \frac{\pi_{\theta}(o_i \mid I, q)}{\pi_{\theta_{\mathrm{old}}}(o_i \mid I, q)} \hat{A}_i,\, \mathrm{clip}\big(\frac{\pi_{\theta}(o_i \mid I, q)}{\pi_{\theta_{\mathrm{old}}}(o_i \mid I, q)}, 1-\epsilon, 1+\epsilon\big) \hat{A}_i \right) \right],
\end{equation}
where $\epsilon$ controls the clipping range. Following recent practices~\citep{meng2025mm, liu2025understanding}, we omit the explicit KL divergence penalty against a reference model. While highly effective, standard GRPO normalizes rewards strictly within a single visual context $I$, leaving the model blind to cross-view inconsistencies.

\subsection{From Independent Augmentation to View-Coupled Credit Assignment}
A straightforward use of visual transformations in GRPO is to treat transformed images as additional augmented samples. 
Although this increases input coverage, it preserves the pointwise objective: $I$ and its transformed view $\tilde I=T_\alpha(I)$ are evaluated in separate rollout groups, so correctness is rewarded per view rather than across views. 
Consequently, a response that is correct on $I$ but unstable under $\tilde I$ can still be reinforced, because the two outcomes are never compared under a shared reward baseline.
Agreement alone is not sufficient either, since the model can be consistently wrong across views. 
ConsistRoll therefore introduces a correctness-gated, view-coupled signal: semantically invariant views are placed in the same GRPO group, their rewards are normalized jointly, and agreement is rewarded only when both predictions are correct. 
In this way, cross-view invariance becomes a training bias in the policy update, rather than an incidental effect of data augmentation.

\subsection{Learning with Cross-View Invariance as a Training Bias}
\label{sec:consistroll}
For a semantically-invariant sample $(I,q,y^*)\in\mathcal S_{\mathrm{inv}}$, ConsistRoll constructs a rotated view $\tilde I=T_\alpha(I)$ that preserves the question-relevant answer. 
We denote the two visual contexts as $x_0=(I,q)$ and $x_1=(\tilde I,q)$. 
Under the same total rollout budget as GRPO, the old policy samples $n/2$ completions from each view, i.e., $o_{v,i}\sim\pi_{\theta_{\mathrm{old}}}(\cdot\mid x_v)$ for $v\in\{0,1\}$ and $i\in[n/2]$. 
The resulting $n$ completions are treated as one coupled GRPO group, so the original and rotated views share the same reward statistics rather than forming two independent normalization groups.

Let $\hat a(\cdot)$ be the answer parser and let $\kappa(o_{0,i},o_{1,i})$ indicate whether the paired completions give the same parsed answer. 
For each paired index $i$, ConsistRoll assigns the following correctness-gated consistency reward to both views:
\begin{equation}
\label{eq:consistroll-reward}
r_{v,i}=
\mathbbm{1}\!\left[\hat a(o_{v,i})=y^*\right]+
\lambda\,
\mathbbm{1}\!\left[
\hat a(o_{0,i})=y^*
\land
\kappa(o_{0,i},o_{1,i})=1
\right],
\quad
v\in\{0,1\},\; i\in[n/2].
\end{equation}
The first term is the standard verifiable correctness reward. 
The second term is a shared paired bonus: it is activated only when both completions are correct and answer-consistent. 
This conjunctive design rules out the degenerate shortcut of producing the same incorrect answer across views, while still encouraging the policy to preserve its answer under semantically-invariant transformations.

Given the coupled reward group $\mathcal C=\{{r_{v,i}\mid v\in{0,1},i\in[n/2]}$\}, ConsistRoll computes normalized advantages over all $n$ rollouts and optimizes the same clipped surrogate form as GRPO:
\begin{equation}
\label{eq:consistroll-objective}
\mathcal{J}(\theta)=
\mathbb{E}\Bigg[\frac{1}{n}\sum_{v=0}^{1}\sum_{i=1}^{n/2}\min\left(
\frac{\pi_{\theta}(o_{v,i}\mid x_v)}
{\pi_{\theta_{\mathrm{old}}}(o_{v,i}\mid x_v)}
\tilde A_{v,i},
\mathrm{clip}\big(
\frac{\pi_{\theta}(o_{v,i}\mid x_v)}
{\pi_{\theta_{\mathrm{old}}}(o_{v,i}\mid x_v)},
1-\epsilon,
1+\epsilon
\big)
\tilde A_{v,i}
\right)\Bigg],
\end{equation}
where
$\tilde A_{v,i}=(r_{v,i}-\mu_{\mathcal C})/(\sigma_{\mathcal C}+\delta)$, 
$\mu_{\mathcal C}=\frac{1}{n}\sum_{v=0}^{1}\sum_{i=1}^{n/2}r_{v,i}$, and 
$\sigma_{\mathcal C}$ is the standard deviation of the same coupled reward group. 
ConsistRoll keeps the GRPO optimizer unchanged, but shifts credit assignment from a single visual context to an answer-preserving view pair. 
Rollouts are jointly normalized with their rotated-view counterparts and receive additional advantage only when both paired responses are correct and consistent. 
For non-invariant samples $(I,q,y^*)\in\mathcal S_{\mathrm{ninv}}$, ConsistRoll disables the rotated branch and exactly recovers standard single-view GRPO with $n$ rollouts.
Algorithm~\ref{alg:consistroll} summarizes the full training procedure of the proposed ConsistRoll.

\begin{algorithm}[t]
\caption{ConsistRoll: Consistency-aware Reinforcement Fine-Tuning}
\label{alg:consistroll}
\begin{algorithmic}[1]
\State \textbf{Input:} Policy $\pi_{\theta}$, old policy $\pi_{\theta_{\mathrm{old}}}$, dataset $p_{\mathcal D}$, training steps $t_{\max}$, rollout number $n$, clip range $\epsilon$, consistency weight $\lambda$, stability constant $\delta$, semantically-invariant transformation $T_\alpha(\cdot)$.
\For{$t=1$ to $t_{\max}$}
    \State Sample a batch of training instances $(I,q,y^*)\sim p_{\mathcal D}$.
    \For{each $(I,q,y^*)$ in the batch}
        \If{$(I,q,y^*)\in\mathcal S_{\mathrm{inv}}$ (semantically-invariant)}
            \State Construct the transformed view $\tilde I=T_\alpha(I)$.
            \State Sample paired rollouts $o_{0,i}\sim\pi_{\theta_{\mathrm{old}}}(\cdot\mid I,q)$, $o_{1,i}\sim\pi_{\theta_{\mathrm{old}}}(\cdot\mid\tilde I,q)$ for $i\in[n/2]$.
            \State Compute $r^a_{v,i}=\mathbbm{1}[\hat a(o_{v,i})=y^*]$ for $v\in\{0,1\}$.            \State Compute $b_i=\mathbbm{1}[a(o_{0,i})=y^*\land\kappa(o_{0,i},o_{1,i})=1]$.
            \State Set $r_{v,i}=r^a_{v,i}+\lambda b_i$ for $v\in\{0,1\}$.
            \State Normalize $\{r_{0,i},r_{1,i}\}_{i=1}^{n/2}$ jointly to obtain view-coupled advantages.
        \Else
            \State Sample single-view rollouts $o_i\sim\pi_{\theta_{\mathrm{old}}}(\cdot\mid I,q)$ for $i\in[n]$.
            \State Compute $r_i=\mathbbm{1}[\hat a(o_i)=y^*]$.
            \State Normalize $\{r_i\}_{i=1}^{n}$ to obtain standard GRPO advantages.
        \EndIf
    \EndFor
    \State Update $\pi_\theta$ by maximizing the clipped ConsistRoll objective in Eq.~\eqref{eq:consistroll-objective}.
    \State Update $\theta_{\mathrm{old}}\leftarrow\theta$.
\EndFor
\end{algorithmic}
\end{algorithm}

\subsection{Theoretical Analysis}
\label{sec:theory}

We analyze whether cross-view consistency is a \emph{valid and learnable inductive bias} for RLVR-trained MLLMs. 
Here, validity means that the bias is aligned with the answer-preserving target and does not introduce a wrong-but-consistent optimum; learnability means that the bias appears as an estimable policy-gradient signal under RL optimization. 
For clarity, the analysis uses the unclipped population objective; the implemented algorithm keeps the clipped GRPO surrogate in Eq.~\eqref{eq:consistroll-objective}. Full derivations, finite-sample advantage analysis, and weaker distributional compatibility are given in Appendix~\ref{app:proofs}.
\begin{assumption}[Semantically-invariant views]
\label{asm:semantic-invariance}
For each sample $(I_i,q_i,y_i^*)$, ConsistRoll uses views
$\mathcal V_i=\{T_{\alpha_m}(I_i)\}_{m=1}^{K}$ that preserve the target answer:
\begin{equation}
    f^*(T_{\alpha_m}(I_i),q_i)=y_i^*,\quad \forall m\in[K].
\end{equation}
\end{assumption}

For a view pair $x_0=(I,q)$ and $x_1=(T_\alpha(I),q)$, let
$c_v(o)=\mathbbm{1}[\hat a(o)=y^*]$ and let $\kappa(o_0,o_1)$ indicate parsed-answer agreement.
ConsistRoll defines
\begin{equation}
\label{eq:theory-bonus}
    b(o_0,o_1)=c_0(o_0)c_1(o_1)\mathbbm{1}[\kappa(o_0,o_1)=1],
    \quad o_v\sim\pi_\theta(\cdot\mid x_v),
\end{equation}
which rewards only paired completions that are both correct and answer-consistent.
\begin{proposition}[Cross-view consistency as a valid and learnable bias]
\label{prop:view-coupled-policy-improvement}
Define the unclipped ConsistRoll population objective
\begin{equation}
\label{eq:cr-main-objective}
    \mathcal J_{\mathrm{ConsistRoll}}(\theta)
    =
    \mathbb E_{o_0,o_1\sim\pi_\theta}
    \left[
    \frac{1}{2}\big(c_0(o_0)+c_1(o_1)\big)+\lambda b(o_0,o_1)
    \right].
\end{equation}
Assume that the parser canonicalizes answers and define 
$p_v(\theta)=\Pr_{o\sim\pi_\theta(\cdot\mid x_v)}[\hat a(o)=y^*]$. 
Then
\begin{equation}
\label{eq:cr-validity-objective}
    \mathcal J_{\mathrm{CR}}(\theta)
    =
    \frac{1}{2}\big(p_0(\theta)+p_1(\theta)\big)+\lambda p_0(\theta)p_1(\theta),
    \;
    \frac{\partial \mathcal J_{\mathrm{CR}}}{\partial p_v}=\frac{1}{2}+\lambda p_{1-v}>0 .
\end{equation}
Thus, at the probability level, the consistency bias is aligned with the answer-preserving target: if the policy class can realize $p_0=p_1=1$, the global maximizer is attained when both views are correct, while wrong-but-consistent answers receive no bonus. 
The bias is learnable through the RL update as
\begin{equation}
\label{eq:cr-gradient-correction}
    \nabla_\theta\mathcal J_{\mathrm{CR}}(\theta)
    =
    \nabla_\theta\mathcal J_{\mathrm{DA}}(\theta)
    +
    \lambda\left(
    p_1(\theta)\nabla_\theta p_0(\theta)
    +
    p_0(\theta)\nabla_\theta p_1(\theta)
    \right),
    \quad
    \mathcal J_{\mathrm{DA}}=\frac{p_0+p_1}{2}.
\end{equation}
If $\mathcal J_{\mathrm{CR}}$ is $L$-smooth and upper bounded, and stochastic gradient ascent uses an unbiased policy-gradient estimator $g_t$ with variance at most $\sigma^2$ and step size $\eta\leq 1/L$, then
\begin{equation}
\label{eq:cr-convergence}
    \frac{1}{T}\sum_{t=0}^{T-1}
    \mathbb E\!\left[\|\nabla_\theta\mathcal J_{\mathrm{CR}}(\theta_t)\|^2\right]
    \leq
    \frac{2(\mathcal J_{\mathrm{CR}}^*-\mathcal J_{\mathrm{CR}}(\theta_0))}{\eta T}
    +L\eta\sigma^2 .
\end{equation}
Choosing $\eta=\mathcal O(T^{-1/2})$ gives the standard $\mathcal O(T^{-1/2})$ convergence rate to first-order stationarity.
\end{proposition}

Proposition~\ref{prop:view-coupled-policy-improvement} shows that cross-view consistency is a valid and learnable inductive bias, not a post-hoc averaging constraint. 
It aligns the optimum with the invariant Bayes target, avoids consistently wrong solutions, and changes RL optimization by amplifying each view's gradient with the success probability of its paired view. 
Thus, semantically-invariant transformations become an explicit cross-view credit-assignment signal, while the induced stochastic optimization still enjoys standard stationary-point convergence under smoothness and bounded-variance assumptions.

\renewcommand{\multirowsetup}{\centering}
\definecolor{mygray}{gray}{.92}
\definecolor{mygreen1}{RGB}{253, 244, 244}
\definecolor{mygreen2}{RGB}{238, 243, 243}
\definecolor{ForestGreen}{RGB}{34,139,34}
\definecolor{Forestred}{RGB}{220,50,50}
\newcommand{\fg}[1]{\mathbf{\mathcolor{ForestGreen}{#1}}}
\newcommand{\fr}[1]{\mathbf{\mathcolor{Forestred}{#1}}}

\newcommand{\cc}{\cellcolor{gray!10}}

\begin{table}[t]
    \centering \vspace{-1.5em}
    \setlength{\tabcolsep}{3.4pt}
    \footnotesize
    \caption{Results on mathematical, logical, and part of general reasoning benchmarks. ConsistRoll consistently improves GRPO and Hint-GRPO across both 3B and 7B backbones. Colored numbers denote absolute changes over the baseline, where $\fr{red}$ denotes decrease, and $\fg{green}$ denotes increase.}
    \vspace{0.5em}
    \label{tab1:main}
    \resizebox{\linewidth}{!}{
    \begin{tabular}{l *{12}{>{\centering\arraybackslash}p{2.8em}}}
        \toprule[1.25pt]
        \multirow{2}{*}{\textbf{Methods}}  
        & \multirow{2}{*}{\makecell[l]{\textbf{Math-}\\ \textbf{Verse}}} 
        & \multirow{2}{*}{\makecell[l]{\textbf{Math-}\\ \textbf{Vision}}} 
        & \multirow{2}{*}{\makecell[l]{\textbf{Math-}\\ \textbf{Vista}}} 
        & \multirow{2}{*}{\makecell[l]{\textbf{Logic-}\\ \textbf{Vista}}} 
        & \multirow{2}{*}{\makecell[c]{\textbf{MMLU}\\ \textbf{Math}}} 
        & \multirow{2}{*}{\makecell[c]{\textbf{JM$^{3}$U}\\ \textbf{-Math}}} 
        & \multicolumn{2}{c}{\textbf{MMStar}} 
        & \multirow{2}{*}{\makecell[c]{\textbf{Dyna-}\\ \textbf{Math}}} 
        & \multirow{2}{*}{\makecell[c]{\textbf{Math}\\ \textbf{Avg.}}} 
        & \multicolumn{2}{c}{\textbf{JMMMU Pro}} \\
        \cmidrule[0.5pt](lr){8-9} 
        \cmidrule[0.5pt](lr){12-13} 
        & & & & & & & \textbf{Logical} & \textbf{Math} & & & \textbf{Std.} & \textbf{Vision}\\
        \midrule
        Qwen2.5-VL-3B 
        & 29.8 & 23.0 & 63.6 & 37.5 & 44.5 & 33.3 & 54.1 & 61.0 & 32.4 
        & 42.1 & 42.6 & 38.8\\
        \midrule
        + SFT 
        & 28.1 & 14.1 & 60.7 & 28.8 & 33.3 & 40.0 & 55.2 & 43.2 & 33.8 
        & {37.5{\tiny$\;\fr{4.6}$}} 
        & 35.4 & 34.5\\
        + GRPO \cite{shao2024deepseekmath} 
        & 31.4 & 21.1 & 64.6 & 36.8 & 46.7 & 36.7 & 53.7 & 63.0 & 33.6 
        & {43.1{\tiny$\;\fg{1.0}$}} 
        & 42.2 & 38.1\\
        \cc + ConsistRoll \scriptsize{(Ours)} 
        & \cc 32.4 & \cc 24.7 & \cc 64.4 & \cc 36.6 & \cc 47.2 & \cc 36.7 
        & \cc 56.1 & \cc 58.5 & \cc 34.5 
        & \cc {43.5{\tiny$\;\fg{1.4}$}} 
        & \cc 43.2 & \cc 39.9\\
        + Hint-GRPO \cite{huang2025boosting} 
        & 32.0 & 21.7 & 65.2 & 36.4 & 46.9 & 36.7 & 56.7 & 61.8 & 32.3 
        & {43.3{\tiny$\;\fg{1.2}$}} 
        & 39.5 & 38.3\\
        \cc + ConsistRoll \scriptsize{(Ours)} 
        & \cc 32.3 & \cc 22.4 & \cc 65.2 & \cc 36.6 & \cc 47.6 & \cc 43.3 
        & \cc 58.5 & \cc 63.7 & \cc 36.8 
        & \cc {45.2{\tiny$\;\fg{3.1}$}} 
        & \cc 42.3 & \cc 38.7\\
        \midrule
        Qwen2.5-VL-7B 
        & 52.4 & 26.6 & 68.6 & 42.0 & 51.0 & 36.7 & 61.0 & 66.3 & 47.9 
        & 50.3 & 46.1 & 45.2\\
        \midrule
        + SFT 
        & 34.3 & 13.2 & 64.7 & 27.7 & 37.8 & 33.3 & 58.9 & 59.8 & 43.9 
        & {41.5{\tiny$\;\fr{8.8}$}} 
        & 41.7 & 42.7\\
        + GRPO \cite{shao2024deepseekmath} 
        & 53.6 & 27.6 & 68.0 & 42.0 & 52.5 & 33.3 & 60.8 & 65.5 & 51.9 
        & {50.6{\tiny$\;\fg{0.3}$}} 
        & 46.3 & 46.4\\
        \cc + ConsistRoll \scriptsize{(Ours)} 
        & \cc 54.2 & \cc 28.9 & \cc 69.2 & \cc 45.8 & \cc 53.2 & \cc 40.0 
        & \cc 63.4 & \cc 67.8 & \cc 52.6 
        & \cc {52.8{\tiny$\;\fg{2.5}$}} 
        & \cc 46.7 & \cc 46.4\\
        + Hint-GRPO \cite{huang2025boosting} 
        & 53.6 & 28.6 & 68.5 & 43.8 & 52.8 & 30.0 & 63.7 & 66.4 & 48.0 
        & {50.6{\tiny$\;\fg{0.3}$}} 
        & 46.2 & 46.3\\
        \cc + ConsistRoll \scriptsize{(Ours)} 
        & \cc 54.4 & \cc 30.6 & \cc 69.0 & \cc 46.7 & \cc 53.7 & \cc 46.7 
        & \cc 65.2 & \cc 66.9 & \cc 50.4 
        & \cc {53.7{\tiny$\;\fg{3.4}$}} 
        & \cc 47.3 & \cc 45.5\\
        \midrule
        Qwen3-VL-4B 
        & 46.8 & 51.6 & 73.7 & 53.2 & 48.2 & 38.5 & 69.8 & 65.3 & 58.1 
        & 56.1 & 48.0 & 46.5\\
        \midrule
        + SFT 
        & 35.6 & 31.4 & 68.2 & 39.5 & 36.8 & 35.0 & 64.2 & 49.5 & 48.2 
        & {45.4{\tiny$\;\fr{10.7}$}} 
        & 42.5 & 41.2\\
        + GRPO \cite{shao2024deepseekmath} 
        & 47.2 & 52.0 & 74.3 & 53.5 & 50.6 & 40.0 & 70.2 & 66.5 & 60.1 
        & {57.1{\tiny$\;\fg{1.1}$}} 
        & 48.5 & 47.2\\
        \cc + ConsistRoll \scriptsize{(Ours)} 
        & \cc 48.5 & \cc 54.3 & \cc 75.6 & \cc 55.4 & \cc 51.2 & \cc 42.3 
        & \cc 71.8 & \cc 67.9 & \cc 61.2 
        & \cc {58.6{\tiny$\;\fg{2.5}$}} 
        & \cc 49.2 & \cc 47.9\\
        + Hint-GRPO \cite{huang2025boosting} 
        & 47.0 & 53.1 & 74.8 & 54.0 & 50.8 & 38.3 & 71.5 & 67.1 & 59.4 
        & {57.3{\tiny$\;\fg{1.2}$}} 
        & 48.4 & 47.0\\
        \cc + ConsistRoll \scriptsize{(Ours)} 
        & \cc 49.1 & \cc 55.4 & \cc 75.1 & \cc 56.2 & \cc 51.8 & \cc 45.0 
        & \cc 72.4 & \cc 67.3 & \cc 60.5 
        & \cc {59.3{\tiny$\;\fg{3.2}$}} 
        & \cc 49.7 & \cc 47.4\\
        \bottomrule[1.25pt]
    \end{tabular}
    }\vspace{-1.5em}
\end{table}

\section{Experiments}
\label{sec:experiments}

\subsection{Experimental Setup}

\textbf{Implementation and training data.}
Our training on geometry-centered multimodal reasoning, where answer-preserving visual transformations are naturally defined and verifiable rewards can be reliably computed. We use Qwen2.5-VL-3B and Qwen2.5-VL-7B as backbone models. Following standard RLVR practice, the model is prompted to generate responses in the format of [\texttt{<think>Reasoning Steps</think>}, \texttt{<answer>Reasoning Result</answer>}]. We train for one epoch on 8,082 math\&logic-related samples extracted from LLaVA-CoT \cite{xu2025llava-cot}, using AdamW with a learning rate of $5\times10^{-5}$. We set the total rollout number to $n=4$. Training is conducted on 8 H20 GPUs with DeepSpeed ZeRO-3, vLLM is used for rollout generation, with one GPU allocated to generation and seven GPUs to policy training. Unless otherwise specified, we set the consistency weight to $\lambda=0.5$. In practice, we use a 90° rotation (or -90°) as the default transformation, since it naturally preserves the answer in geometric reasoning tasks while transforming the visual features.

\textbf{Baselines.}
We compare against the base model, an SFT baseline, standard GRPO~\citep{shao2024deepseekmath}, and Hint-GRPO~\citep{huang2025boosting}. ConsistRoll is applied as a plug-and-play extension to both GRPO and Hint-GRPO. Our design isolates whether cross-view consistency provides gains beyond pointwise RLVR and hint-enhanced reward shaping. All methods use the same backbone, training data, answer parser, and response format, and ConsistRoll changes only the rollout grouping and reward computation.

\begin{table}[t]
    \centering \vspace{-1.5em}
    \setlength{\tabcolsep}{3.5pt}
    \footnotesize
    \caption{
Results on general and hallucination benchmarks. 
ConsistRoll improves both general and hallucination-related accuracy, showing that cross-view consistency provides a robust training signal. 
}
    \vspace{0.25em}
    \label{tab2:main}
    \resizebox{\linewidth}{!}{
    \begin{tabular}{l  *{12}{>{\centering\arraybackslash}p{2.8em}}}
        \toprule[1.25pt]
        \multirow{2}{*}{\textbf{Methods}}  
        & \multirow{2}{*}{\makecell[c]{\textbf{GQA}}} 
        & \multirow{2}{*}{\makecell[c]{\textbf{MMB}}} 
        & \textbf{MMStar} 
        & \multicolumn{4}{c}{\textbf{LLaVA in the wild}} 
        & \multirow{2}{*}{\textbf{\makecell{General\\Avg.}}}
        & \multicolumn{3}{c}{\textbf{HallusionBench}} 
        & \multirow{2}{*}{\textbf{\makecell{Hallu. \\ Avg.}}} \\
        \cmidrule[0.5pt](lr){4-4}  
        \cmidrule[0.5pt](lr){5-8} 
        \cmidrule[0.5pt](lr){10-12}
        & & & \textbf{Avg.} 
        & \textbf{all} & \textbf{complex} & \textbf{conv} & \textbf{detail} 
        & 
        & \textbf{aAcc} & \textbf{fAcc} & \textbf{qAcc} 
        & \\
        \midrule
        Qwen2.5-VL-3B 
        & 60.0 & 77.1 & 55.8 & 65.6 & 65.4 & 84.9 & 45.3 
        & 64.9 
        & 57.7 & 33.0 & 30.1 
        & 40.3 \\
        \midrule
        + SFT 
        & 58.8 & 76.9 & 51.4 & 62.0 & 56.2 & 78.6 & 51.9 
        & {62.3{\tiny$\;\fr{2.6}$}} 
        & 51.2 & 28.9 & 26.1 
        & {35.4{\tiny$\;\fr{4.9}$}} \\
        + GRPO \cite{shao2024deepseekmath} 
        & 60.5 & 77.9 & 57.0 & 64.4 & 67.9 & 71.9 & 44.2 
        & {63.4{\tiny$\;\fr{1.5}$}} 
        & 57.5 & 33.8 & 31.2 
        & {40.8{\tiny$\;\fg{0.5}$}} \\
        \cc + ConsistRoll \scriptsize{(Ours)} 
        & \cc 60.4 & \cc 77.1 & \cc 56.0 & \cc 64.6 & \cc 61.3 & \cc 86.8 & \cc 46.5 
        & \cc {64.7{\tiny$\;\fr{0.2}$}} 
        & \cc 58.6 & \cc 33.2 & \cc 32.0 
        & \cc {41.3{\tiny$\;\fg{1.0}$}} \\
        + Hint-GRPO \cite{huang2025boosting} 
        & 60.3 & 77.1 & 56.8 & 63.7 & 63.4 & 82.4 & 42.2 
        & {63.7{\tiny$\;\fr{1.2}$}} 
        & 57.7 & 34.1 & 32.3 
        & {41.4{\tiny$\;\fg{1.1}$}} \\
        \cc + ConsistRoll \scriptsize{(Ours)} 
        & \cc 60.3 & \cc 77.7 & \cc 57.5 & \cc 66.9 & \cc 67.6 & \cc 81.8 & \cc 48.6 
        & \cc {65.8{\tiny$\;\fg{0.9}$}} 
        & \cc 58.4 & \cc 34.7 & \cc 32.3 
        & \cc {41.8{\tiny$\;\fg{1.5}$}} \\
        \midrule
        Qwen2.5-VL-7B 
        & 60.8 & 82.8 & 62.1 & 75.8 & 85.9 & 83.0 & 51.9 
        & 71.8 
        & 62.7 & 36.4 & 36.0 
        & 45.0 \\
        \midrule
        + SFT 
        & 59.0 & 81.7 & 59.5 & 68.7 & 72.0 & 81.4 & 48.1 
        & {67.2{\tiny$\;\fr{4.6}$}} 
        & 44.5 & 24.6 & 18.2 
        & {29.1{\tiny$\;\fr{16.}$}} \\
        + GRPO \cite{shao2024deepseekmath} 
        & 61.0 & 83.2 & 61.5 & 68.3 & 72.7 & 78.8 & 49.3 
        & {67.8{\tiny$\;\fr{4.0}$}} 
        & 63.4 & 38.4 & 37.1 
        & {46.3{\tiny$\;\fg{1.3}$}} \\
        \cc + ConsistRoll \scriptsize{(Ours)} 
        & \cc 60.8 & \cc 83.2 & \cc 63.2 & \cc 76.1 & \cc 82.4 & \cc 85.6 & \cc 54.9 
        & \cc {72.3{\tiny$\;\fg{0.5}$}} 
        & \cc 63.9 & \cc 39.6 & \cc 38.2 
        & \cc {47.3{\tiny$\;\fg{2.3}$}} \\
        + Hint-GRPO \cite{huang2025boosting} 
        & 60.9 & 83.0 & 62.6 & 71.6 & 87.9 & 73.9 & 44.1 
        & {69.1{\tiny$\;\fr{2.7}$}} 
        & 55.6 & 32.5 & 36.6 
        & {41.6{\tiny$\;\fr{3.4}$}} \\
        \cc + ConsistRoll \scriptsize{(Ours)} 
        & \cc 61.1 & \cc 83.3 & \cc 63.4 & \cc 76.3 & \cc 86.5 & \cc 83.7 & \cc 52.2 
        & \cc {72.4{\tiny$\;\fg{0.6}$}} 
        & \cc 63.8 & \cc 39.6 & \cc 37.8 
        & \cc {47.1{\tiny$\;\fg{2.1}$}} \\
        \midrule
        Qwen3-VL-4B (Instruct) 
        & 62.5 & 85.2 & 69.8 & 81.2 & 89.4 & 87.5 & 55.3 
        & 75.8 
        & 67.5 & 41.2 & 40.4 
        & 49.7 \\
        \midrule
        + SFT 
        & 60.2 & 83.5 & 64.1 & 74.3 & 76.8 & 83.0 & 51.5 
        & {71.1{\tiny$\;\fr{4.7}$}} 
        & 51.3 & 29.8 & 24.5 
        & {35.2{\tiny$\;\fr{14.}$}} \\
        + GRPO \cite{shao2024deepseekmath} 
        & 62.8 & 85.5 & 69.0 & 75.4 & 78.2 & 82.3 & 53.0 
        & {72.3{\tiny$\;\fr{3.5}$}} 
        & 68.2 & 42.5 & 41.3 
        & {50.7{\tiny$\;\fg{1.0}$}} \\
        \cc + ConsistRoll \scriptsize{(Ours)} 
        & \cc 62.6 & \cc 85.7 & \cc 70.8 & \cc 81.9 & \cc 88.3 & \cc 89.2 & \cc 58.4 
        & \cc {76.6{\tiny$\;\fg{0.8}$}} 
        & \cc 68.9 & \cc 43.8 & \cc 42.6 
        & \cc {51.8{\tiny$\;\fg{2.1}$}} \\
        + Hint-GRPO \cite{huang2025boosting} 
        & 62.6 & 85.4 & 70.0 & 77.1 & 91.0 & 78.5 & 48.2 
        & {73.3{\tiny$\;\fr{2.5}$}} 
        & 61.2 & 36.8 & 39.7 
        & {45.9{\tiny$\;\fr{3.8}$}} \\
        \cc + ConsistRoll \scriptsize{(Ours)} 
        & \cc 62.9 & \cc 85.8 & \cc 71.2 & \cc 82.4 & \cc 90.2 & \cc 88.0 & \cc 56.1 
        & \cc {76.7{\tiny$\;\fg{0.9}$}} 
        & \cc 68.8 & \cc 44.0 & \cc 42.1 
        & \cc {51.6{\tiny$\;\fg{1.9}$}} \\
        \bottomrule[1.25pt]
    \end{tabular}
    }\vspace{-0.5em}
\end{table}
\begin{figure}[t]
    \centering\vspace{-0.25em}
    \begin{subfigure}[b]{0.475\textwidth}
        \centering
        \includegraphics[width=\textwidth]{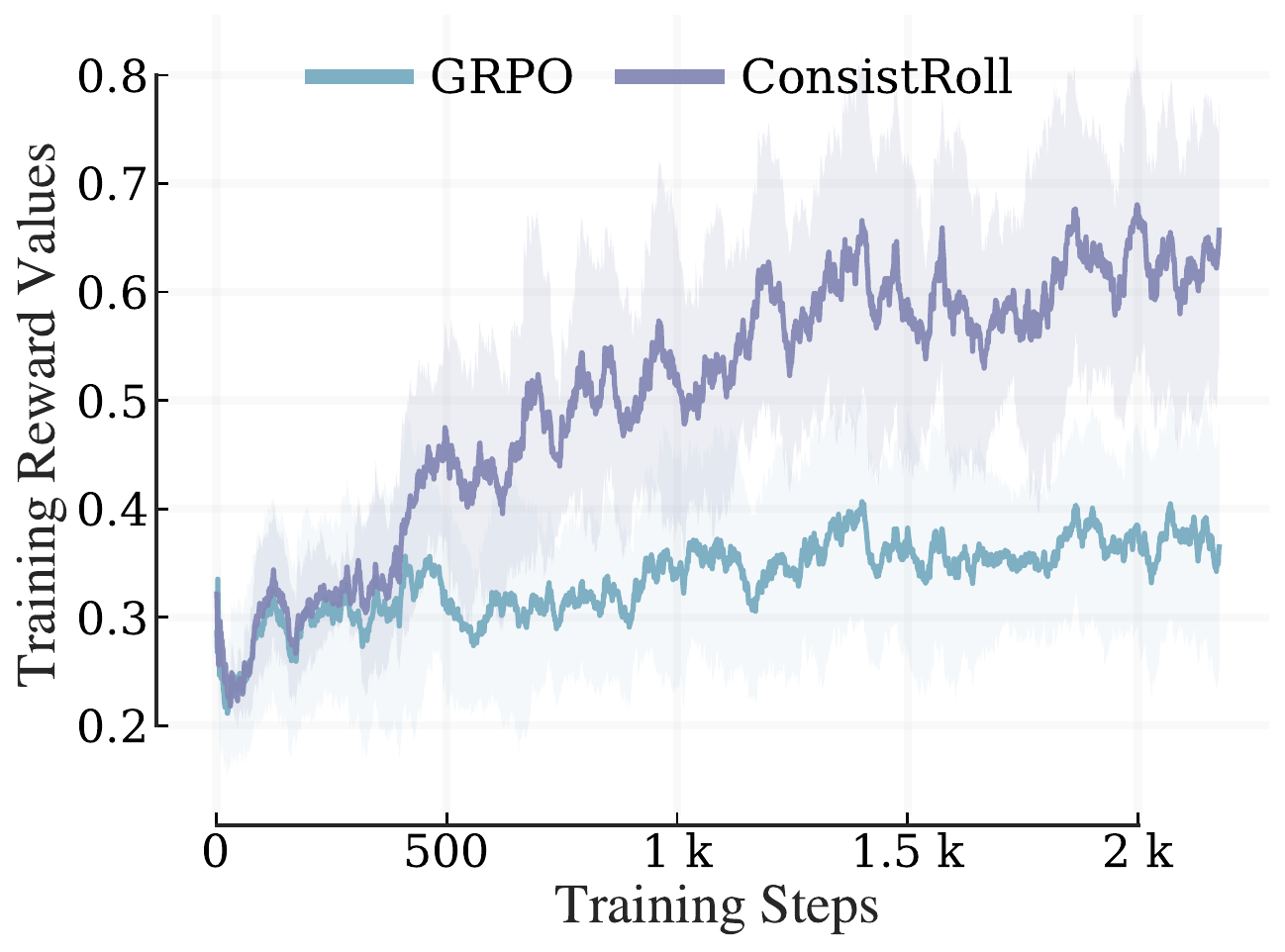}
        \vspace{-1.5em}\caption{Qwen2.5-VL-3B}
        \label{fig:exp_a}
    \end{subfigure}
    \hfill 
    \begin{subfigure}[b]{0.475\textwidth}
        \centering
        \includegraphics[width=\textwidth]{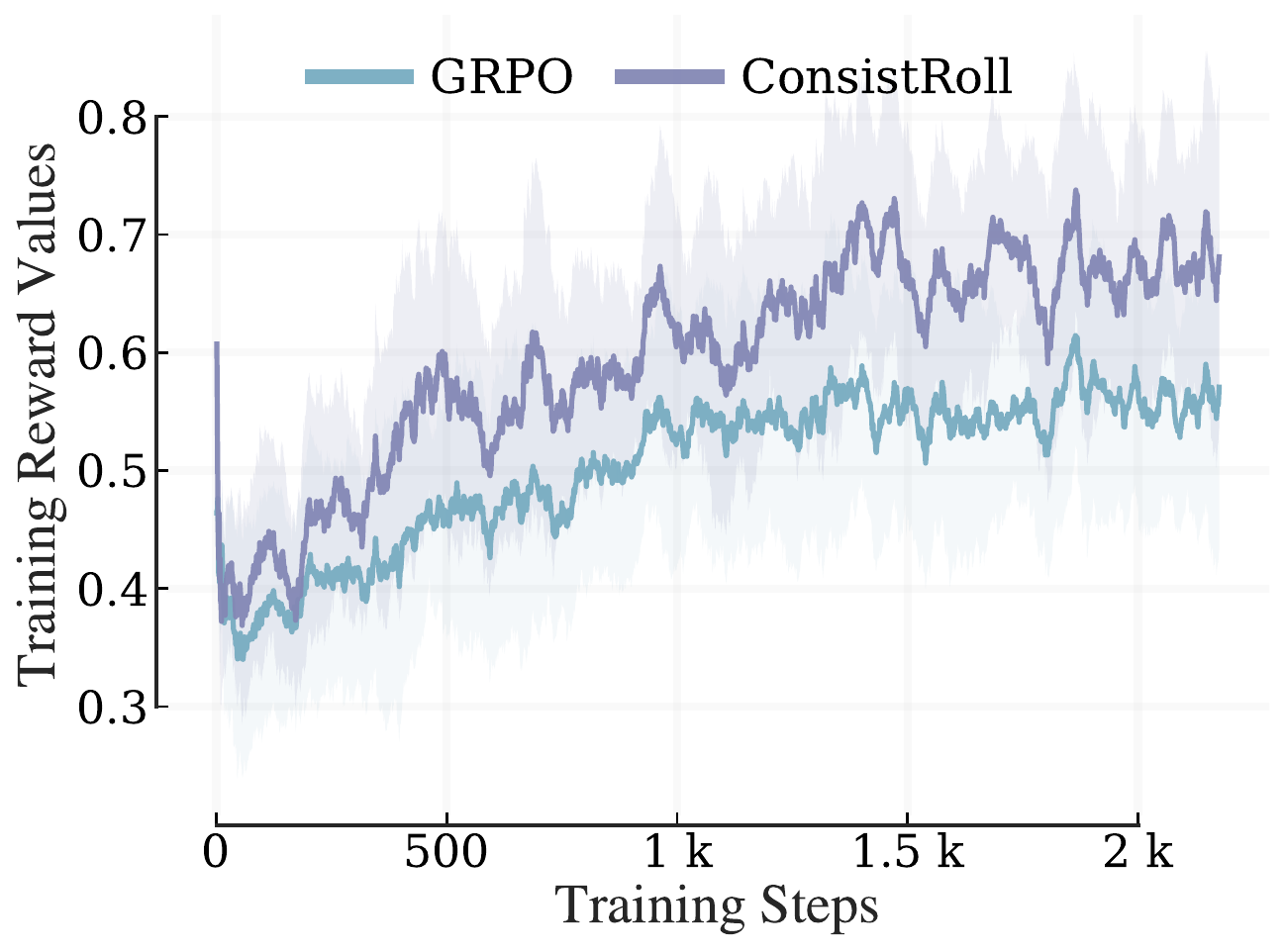}
        \vspace{-1.5em}\caption{Qwen2.5-VL-7B}
        \label{fig:exp_b}
    \end{subfigure} \vspace{-0.5em}
    \caption{Training reward curves of GRPO and our proposed ConsistRoll. 
ConsistRoll achieves consistently higher rewards and more stable optimization across training steps, indicating that view-coupled consistency provides a stronger and more reliable learning signal than pointwise GRPO.}
    \label{fig:reward_curve} \vspace{-1em}
\end{figure}
\textbf{Evaluation.}
We evaluate all models using the LMMs-Eval kit \cite{zhang2025lmms} under a unified evaluation protocol, ensuring that prompts, decoding settings, answer extraction, and metric computation are kept consistent across methods.
We consider 14 benchmark families spanning mathematical reasoning, logical reasoning, general multimodal understanding, and hallucination-sensitive evaluation.
Table~\ref{tab1:main} reports MathVerse, MathVision, MathVista, LogicVista, MMLU-Pro Math, JMMMU(JM3U)-Math, MMStar, DynaMath, and JMMMU-Pro.
Table~\ref{tab2:main} reports GQA, MMBench, MMStar, LLaVA-in-the-Wild, and HallusionBench.
For HallusionBench, we report answer-level accuracy, figure-level accuracy, question-pair accuracy, and their average.
All results are reported as percentage scores.
 {~\footnote{\url{https://github.com/obananas/ConsistRoll}}}

\subsection{Main Results}

\textbf{ConsistRoll improves mathematical and logical reasoning.}
Table~\ref{tab1:main} shows that ConsistRoll consistently improves GRPO-style training on math and logic benchmarks. On Qwen2.5-VL-7B, adding ConsistRoll to GRPO improves the Math Avg. from 50.6 to 52.8. When applied to Hint-GRPO, the gain is larger, increasing the Math Avg. from 50.6 to 53.7 (+3.1). The gains are broad across reasoning-heavy benchmarks: with GRPO on the 7B model, ConsistRoll improves LogicVista from 42.0 to 45.8, MMStar-Logical from 60.8 to 63.4, and MMStar-Math from 65.5 to 67.8. The same trend holds on Qwen2.5-VL-3B, although the margins are smaller under vanilla GRPO. 

\textbf{ConsistRoll transfers to general multimodal understanding.}
Table~\ref{tab2:main} further shows that the benefits are not limited to geometry-style training tasks. On Qwen2.5-VL-7B, ConsistRoll improves the General Avg. of GRPO from 67.8 to 72.3, a substantial +4.5 point gain. The improvement is especially clear on LLaVA-in-the-Wild, where ConsistRoll improves the overall score from 68.3 to 76.1, complex reasoning from 72.7 to 82.4, conversation from 78.8 to 85.6. ConsistRoll also improves Hint-GRPO on the 7B model, increasing the General Avg. from 69.1 to 72.4 (+3.3).

\textbf{ConsistRoll improves hallucination-sensitive evaluation.} As shown in Table~\ref{tab2:main}, 
ConsistRoll also improves HallusionBench performance. On Qwen2.5-VL-7B, ConsistRoll improves the Hallucination Avg. of GRPO from 46.3 to 47.3. When applied to Hint-GRPO, the improvement is more pronounced, increasing the Hallucination Avg. from 41.6 to 47.1 (+5.5). This gain is accompanied by clear improvements in answer-level, figure-level, and question-pair accuracy. On Qwen2.5-VL-3B, ConsistRoll also improves the Hallucination Avg. for both GRPO and Hint-GRPO, from 40.8 to 41.3 and from 41.4 to 41.8, respectively, although some individual submetrics fluctuate.
These results support the motivation of ConsistRoll. By rewarding only completions that are both correct and consistent across answer-preserving views, ConsistRoll strengthens visually grounded reasoning.

\begin{table}[t]
    \centering \vspace{-0.5em}
    \setlength{\tabcolsep}{3.2pt}
    \footnotesize
    \caption{Ablation on the consistency weight $\lambda$ in $r_k=r_k^a+\lambda r_k^c$ on Qwen2.5-VL-7B. $\lambda=0$ keeps paired-view rollouts but removes the consistency reward. $\Delta$ denotes improvement against GRPO.}
    \vspace{0.25em}
    \label{tab:lambda-ablation}
    \resizebox{\linewidth}{!}{
    \begin{tabular}{l*{9}{>{\centering\arraybackslash}p{4em}}}
        \toprule[1.25pt]
        \multirow{2}{*}{\textbf{Settings}} & \multirow{2}{*}{\makecell[c]{\textbf{Math}\\ \textbf{Average}}} & \multirow{2}{*}{\makecell[c]{\textbf{Logic-}\\ \textbf{Vista}}} & \multirow{2}{*}{\makecell[c]{\textbf{General}\\ \textbf{Average}}} & \multicolumn{3}{c}{\textbf{HallusionBench}} & \multirow{2}{*}{\makecell[l]{\textbf{Hallu.}\\ \textbf{Average}}} &\multirow{2}{*}{\makecell[l]{\textbf{Overall}}} & \multirow{2}{*}{\makecell[l]{\textbf{$\Delta$}}} \\
        \cmidrule(lr){5-7}
        & & & & \textbf{aAcc} & \textbf{fAcc} & \textbf{qAcc} & & & \\
        \midrule
        GRPO \cite{shao2024deepseekmath} & 50.57 & 41.96 & 67.82 & 63.41 & 38.44 & 37.14 & 46.33 & 54.91 & -- \\
        \midrule
        w/ Consist-only & 50.29 & 43.75 & 68.11 & 63.93 & 39.60 & 38.24 & 47.26 & 55.22 & +0.31 \\
        \midrule
        $\lambda=0.0$ & 50.65 & 42.41 & 68.48 & 64.04 & 38.73 & 38.02 & 46.93 & 55.35 & +0.44 \\
        $\lambda=0.1$ & 51.54 & 44.64 & 69.87 & 64.46 & 39.31 & \textbf{39.12} & \textbf{47.63} & 56.34 & +1.44 \\
        $\lambda=0.3$ & 51.47 & 44.20 & 71.53 & 63.20 & 37.86 & 37.36 & 46.14 & 56.38 & +1.47 \\
        \rowcolor{gray!10}
        $\lambda=0.5\ (\mathrm{defualt})$ & \textbf{52.79} & 45.76 & \textbf{72.32} & 63.93 & \textbf{39.60} & 38.24 & 47.26 & \textbf{57.45} & \textbf{+2.55} \\
        $\lambda=0.7$ & 51.48 & 45.54 & 70.19 & 63.51 & 38.15 & 38.24 & 46.63 & 56.10 & +1.19 \\
        $\lambda=0.8$ & 50.98 & 42.86 & 69.29 & \textbf{64.77} & 39.02 & 38.90 & 47.56 & 55.95 & +1.04 \\
        $\lambda=1.0$ & 51.34 & \textbf{46.21} & 69.71 & 64.04 & 39.02 & 37.58 & 46.88 & 55.98 & +1.07 \\
        \bottomrule[1.25pt]
    \end{tabular}
    }
    \vspace{-0.8em}
\end{table}

\begin{table}[t]
    \centering \vspace{-0.5em}
    \setlength{\tabcolsep}{3.4pt}
    \footnotesize
    \caption{
Ablation on semantically-invariant transformations. 
We compare ConsistRoll using the default rotation-based and noise perturbation. 
Rotation provides more stable gains across benchmarks
}
    \vspace{0.5em}\label{tab:transformation_ablation}
    \resizebox{\linewidth}{!}{
    \begin{tabular}{l *{12}{>{\centering\arraybackslash}p{2.8em}}}
        \toprule[1.25pt]
        \multirow{2}{*}{\textbf{Methods}}  
        & \multirow{2}{*}{\makecell[l]{\textbf{Math-}\\ \textbf{Verse}}} 
        & \multirow{2}{*}{\makecell[l]{\textbf{Math-}\\ \textbf{Vision}}} 
        & \multirow{2}{*}{\makecell[l]{\textbf{Math-}\\ \textbf{Vista}}} 
        & \multirow{2}{*}{\makecell[l]{\textbf{Logic-}\\ \textbf{Vista}}} 
        & \multirow{2}{*}{\makecell[c]{\textbf{MMLU}\\ \textbf{Math}}} 
        & \multirow{2}{*}{\makecell[c]{\textbf{JM$^{3}$U}\\ \textbf{-Math}}} 
        & \multicolumn{2}{c}{\textbf{MMStar}} 
        & \multirow{2}{*}{\makecell[c]{\textbf{Dyna-}\\ \textbf{Math}}} 
        & \multirow{2}{*}{\makecell[c]{\textbf{Math}\\ \textbf{Avg.}}} 
        & \multicolumn{2}{c}{\textbf{JMMMU Pro}} \\
        \cmidrule[0.5pt](lr){8-9} 
        \cmidrule[0.5pt](lr){12-13} 
        & & & & & & & \textbf{Logical} & \textbf{Math} & & & \textbf{Std.} & \textbf{Vision}\\
        \midrule
        + GRPO \cite{shao2024deepseekmath} 
        & 53.6 & 27.6 & 68.0 & 42.0 & 52.5 & 33.3 & 60.8 & 65.5 & 51.9 
        & {50.6{\tiny$\;\fg{0.3}$}} 
        & 46.3 & 46.4\\
        \cc + ConsistRoll$_{\text{w/ rotate}}$   
        & \cc 54.2 & \cc 28.9 & \cc 69.2 & \cc 45.8 & \cc 53.2 & \cc 40.0 
        & \cc 63.4 & \cc 67.8 & \cc 52.6 
        & \cc {52.8{\tiny$\;\fg{2.5}$}} 
        & \cc 46.7 & \cc 46.4\\
        \cc + ConsistRoll$_{\text{w/ noise}}$
        & \cc 53.1 & \cc 33.2 & \cc 69.1 & \cc 45.5 & \cc 51.2 & \cc 36.7
        & \cc 62.7 & \cc 64.9 & \cc 51.5
        & \cc {52.0{\tiny$\;\fg{1.7}$}}
        & \cc 45.5 & \cc 46.0\\ \midrule
        \multirow{2}{*}{\textbf{Methods}}  
        & \multirow{2}{*}{\makecell[c]{\textbf{GQA}}} 
        & \multirow{2}{*}{\makecell[c]{\textbf{MMB}}} 
        & \textbf{MMStar} 
        & \multicolumn{4}{c}{\textbf{LLaVA in the wild}} 
        & \multirow{2}{*}{\textbf{\makecell{General\\Avg.}}}
        & \multicolumn{3}{c}{\textbf{HallusionBench}} 
        & \multirow{2}{*}{\textbf{\makecell{Hallu. \\ Avg.}}} \\
        \cmidrule[0.5pt](lr){4-4}  
        \cmidrule[0.5pt](lr){5-8} 
        \cmidrule[0.5pt](lr){10-12}
        & & & \textbf{Avg.} 
        & \textbf{all} & \textbf{complex} & \textbf{conv} & \textbf{detail} 
        & 
        & \textbf{aAcc} & \textbf{fAcc} & \textbf{qAcc} 
        & \\
        \midrule
+ GRPO \cite{shao2024deepseekmath} 
        & 61.0 & 83.2 & 61.5 & 68.3 & 72.7 & 78.8 & 49.3 
        & {67.8{\tiny$\;\fr{4.0}$}} 
        & 63.4 & 38.4 & 37.1 
        & {46.3{\tiny$\;\fg{1.3}$}} \\
        \cc + ConsistRoll$_{\text{w/ rotate}}$ 
        & \cc 60.8 & \cc 83.2 & \cc 63.2 & \cc 76.1 & \cc 82.4 & \cc 85.6 & \cc 54.9 
        & \cc {72.3{\tiny$\;\fg{0.5}$}} 
        & \cc 63.9 & \cc 39.6 & \cc 38.2 
        & \cc {47.3{\tiny$\;\fg{2.3}$}} \\ 
        \cc + ConsistRoll$_{\text{w/ noise}}$
        & \cc 60.9 & \cc 83.2 & \cc 60.8 & \cc 73.4 & \cc 78.9 & \cc 83.7 & \cc 52.3
        & \cc {70.4{\tiny$\;\fr{1.3}$}}
        & \cc 62.9 & \cc 36.7 & \cc 36.1
        & \cc {45.2{\tiny$\;\fg{0.2}$}} \\
        \bottomrule[1.25pt]
    \end{tabular}
    }\vspace{-0.5em}
\end{table}
\subsection{Ablation and Optimization Analysis}

\textbf{Effect of the consistency weight.}
Table~\ref{tab:lambda-ablation} studies the consistency weight $\lambda$ in the reward $r_k = r_k^a + \lambda r_k^c$. The default setting $\lambda=0.5$ gives the best overall trade-off, achieving the best performance, and overall Avg. of 57.45, with a +2.55 point gain over GRPO, and also improves LogicVista from 41.96 to 45.76, showing that the selected weight benefits the logical visual reasoning.

The ablation separates the effect of paired-view exposure from the consistency reward. When $\lambda=0$, the model still uses paired original and transformed rollouts, but removes the consistency reward. This setting improves only slightly over GRPO, indicating that transformed-view exposure alone is helpful but insufficient. Moderate positive values of $\lambda$ perform best: small values underuse the cross-view signal, while large values may overemphasize agreement and weaken the pointwise correctness objective. Although $\lambda=0.1$ and $\lambda=0.8$ obtain slightly higher HallusionBench averages, they underperform the default setting $\lambda=0.5$ on the main reasoning and general-purpose averages.

\textbf{Effect of the augmentation.}
We further study whether the choice of semantically-invariant transformation affects ConsistRoll. 
Table~\ref{tab:transformation_ablation} compares the default rotation-based transformation with random noise perturbation on the same GRPO backbone. 
Both variants improve mathematical reasoning over GRPO, confirming that coupling semantically related views is beneficial beyond a specific augmentation type. 
However, the rotation-based variant achieves stronger and more balanced gains, improving Math Avg. from $50.6$ to $52.8$, General Avg. from $67.8$ to $72.3$. 
In contrast, noise augmentation improves Math Avg. to $52.0$ and remains competitive on several visual reasoning subsets, but its gains are less stable on general and hallucination-sensitive benchmarks. 
We think ConsistRoll benefits more from structured transformations that preserve high-level semantics while inducing meaningful view changes, and low-level noise provides less stable consistency signals.

\begin{table}[t!]
\vspace{-0.5em}
\fontsize{8pt}{\baselineskip}\selectfont
\linespread{0.95}\selectfont
\begin{minipage}{0.46\columnwidth}
\begin{mybody}
    \begin{minipage}[s][2.0cm]{0.41\columnwidth}
    \includegraphics[height=1.7cm]{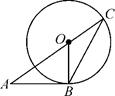}
    \end{minipage}
    \begin{minipage}[s][2.0cm]{0.56\columnwidth}
    \textbf{Question:} As shown in the figure, if angle A = 36.0, then angle C is equal to () \\
    \textbf{Choices:} \\
    A:36° \;   B:54° \;
    C:60°  \;  \textcolor{darkgreen}{\textbf{D:27°}} 
    \end{minipage}
    \vspace{0.5em}\\\textbf{Key reasoning of GRPO:}\\
   1. Identify the given angle: angle A = 36°. \textcolor{darkred}{2. Note that angle A and angle C are circumference angles subtended by the same arc. 3. By the same-arc theorem, equal arcs subtend equal circumference angles. Therefore, angle C = angle A = 36°. }
    \vspace{0.6em}\\ \textbf{Key reasoning of ConsistRoll:}\\
    1. Identify the tangent property: line AB is tangent at B, so radius OB is perpendicular to AB and angle OBA = 90°. 2. Find angle AOB: angle AOB = 180° - 90° - 36° = 54°. 3. Since OB and OC are both radii, triangle OBC is isosceles, so angle OCB equals angle OBC. 4. By the exterior angle property, angle AOB equals twice angle OCB. Therefore, angle OCB = 27°.
\end{mybody}
\vspace{-1mm}
\end{minipage}
 \hfill
\begin{minipage}{0.53\columnwidth}
\begin{mybody}
    \begin{minipage}[s][2.5cm]{0.35\columnwidth}
    \includegraphics[height=2.1cm]{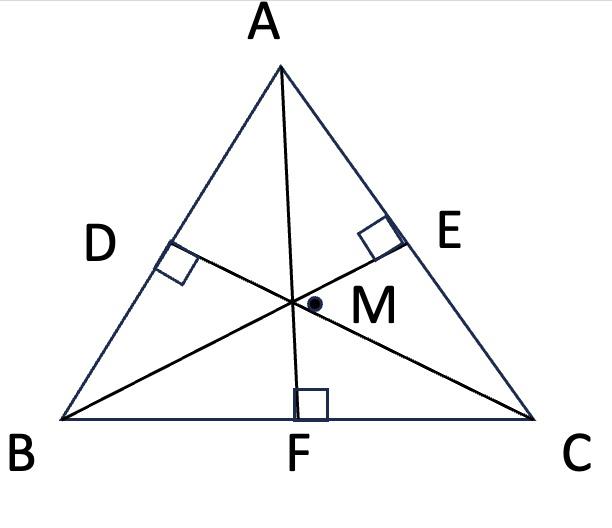}
    \end{minipage}
    \begin{minipage}[s][2.5cm]{0.62\columnwidth}
    \textbf{Question:} As shown in the figure, if BE $\perp$ AC, CD $\perp$ AB, AF $\perp$ BC, then is M the incenter of triangle ABC? \\ Please answer yes or no. \vspace{0.5em}\\ \textbf{Prediction of GRPO:} \textcolor{darkred}{Yes.} \\\textbf{Prediction of ConsistRoll:} No.
    \end{minipage}\vspace{-1.em}
\end{mybody}
\vspace{-3mm}
\begin{mybody}
    \begin{minipage}[s][1.9cm]{1\columnwidth}
    \includegraphics[height=1.3cm]{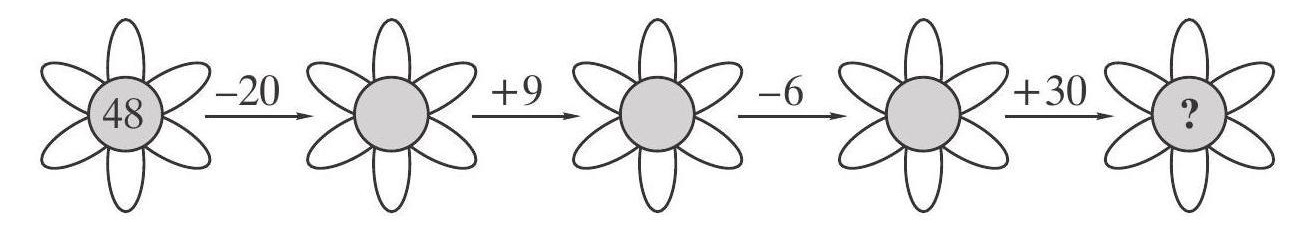}
    \end{minipage} \vspace{-2em}\\
    \begin{minipage}[s][1.9cm]{1\columnwidth}
    \textbf{Question:} Which number do you have to write in the last daisy? 
    \end{minipage}\vspace{-5.5em}
    \\\textbf{Key reasoning of GRPO:}\\
    1. Start at 48. 2. Deduct 20 from 48: (48 - 20 = 28). \textcolor{darkred}{3. Adjust the weight and add 9: (28 + 9 = 35). 4. Subtract 6 from the daisy and add 30 to get 61: (35 - 6 + 30 = 61).} The number is (\boxed{61})
    \\\textbf{Key reasoning of ConsistRoll:}\\
    1. Start with the number 48. 2. Subtract 20 from 48: (48 - 20 = 28). 3. Add 9 to 28: (28 + 9 = 37). 4. Subtract 6 from 37: (37 - 6 = 31). 5. Add 30 to 31: (31 + 30 = 61). The number is (\boxed{61})
\end{mybody}
\vspace{-1mm}
\end{minipage}
\vspace{0.5em}
\caption{Qualitative case studies comparing GRPO and our proposed ConsistRoll. }
\vspace{-2.5em}
\label{fig:cases}
\end{table}

\textbf{Beyond data augmentation.}
The $\lambda=0$ setting corresponds to GRPO+DA: the model sees the same transformed views, but receives no consistency bonus. This control only slightly improves over GRPO, increasing the Overall Avg. from 54.91 to 55.35 (+0.44). By contrast, default ConsistRoll reaches 57.45 (+2.55), with clear gains on math and general average, which show that naive DA is not sufficient. The key factor is the correctness-aware cross-view reward, which converts answer-preserving transformations from passive augmentation into an explicit credit-assignment signal.

\textbf{Training dynamics.}
Figure~\ref{fig:reward_curve} compares the training reward curves of GRPO and ConsistRoll. ConsistRoll reaches higher rewards and exhibits smoother optimization on both Qwen2.5-VL-3B and Qwen2.5-VL-7B, suggesting that the view-coupled group provides a stronger credit-assignment signal than pointwise GRPO. Since ConsistRoll includes an additional consistency reward term, we treat the reward curves as an optimization diagnostic rather than standalone evidence of final performance; the main evidence comes from downstream benchmark improvements in Tables~\ref{tab1:main} and~\ref{tab2:main}.

\vspace{-0.5em}
\subsection{Qualitative Analysis}
\label{sec:qualitative}

\begin{wraptable}{r}{0.45\textwidth}
\vspace{-2em}
\centering
\setlength{\tabcolsep}{4.5pt}
\small
\caption{Cross-view stability on GeoQA.}
\vspace{-0.5em}
\label{tab:cross-view-stability}
\resizebox{\linewidth}{!}{
\begin{tabular}{lcccc}
\toprule
\textbf{Model} 
& \makecell[c]{$D_{\mathrm{rot}}$\\\textbf{Deg.}} 
& \makecell[c]{$I_{\mathrm{rot}}$\\\textbf{Imp.}} 
& \makecell[c]{$F_{\mathrm{rot}}$\\\textbf{Flips}} 
& \makecell[c]{$C_{\mathrm{rot}}$\\\textbf{(\%)}} \\
\midrule
Qwen2.5-VL-3B & 78 & 57 & 135 & 82.10 \\
GRPO           & 65 & 55 & 120 & 84.08 \\
ConsistRoll    & 43 & 51 &  94 & 87.53 \\
\midrule
Qwen3-VL-4B   & 58 & 46 & 104 & 86.42 \\
GRPO           & 48 & 43 &  91 & 88.15 \\
ConsistRoll    & 22 & 35 &  57 & 91.85 \\
\bottomrule
\end{tabular}}
\vspace{-1.em}
\end{wraptable}

\textbf{Cross-view stability.}
We evaluate prediction stability on $754$ original-rotation pairs from GeoQA \cite{chen2021geoqa}, using degradation $D_{\mathrm{rot}}$, improvement $I_{\mathrm{rot}}$, total flips $F_{\mathrm{rot}}$, and rotation consistency $C_{\mathrm{rot}}$; detailed definitions are provided in Appendix~\ref{app:cross-view-stability}. 
As shown in Table~\ref{tab:cross-view-stability}, ConsistRoll improves $C_{\mathrm{rot}}$ from $82.10\%$ for Qwen2.5-VL-3B and $84.08\%$ for GRPO to $87.53\%$, while reducing total flips from $135$ and $120$ to $94$. Rotation-induced degradation drops indicate that view-coupled credit assignment better preserves correct predictions under the view transformations.

\textbf{Case study.} Figure~\ref{fig:cases} shows examples where our proposed ConsistRoll better follows visual evidence.




\vspace{-0.5em}\section{Conclusion}
We presented ConsistRoll, a simple RLVR method that turns cross-view consistency into an online credit-assignment signal. By grouping answer-preserving views in GRPO and rewarding paired completions only when they are both correct and consistent, ConsistRoll improves Qwen2.5-VL-3B/7B across reasoning, general, and hallucination-sensitive benchmarks. Ablations show that correctness-aware consistency, rather than augmentation alone, is the main source of improvement.



\clearpage
{\small
\bibliographystyle{plain}
\bibliography{reference}

\begin{thebibliography}{10}

\bibitem{bai2025qwen3}
Shuai Bai, Yuxuan Cai, Ruizhe Chen, Keqin Chen, Xionghui Chen, Zesen Cheng, Lianghao Deng, Wei Ding, Chang Gao, Chunjiang Ge, et~al.
\newblock Qwen3-vl technical report.
\newblock {\em arXiv preprint arXiv:2511.21631}, 2025.

\bibitem{bronstein2021geometric}
Michael~M. Bronstein, Joan Bruna, Taco Cohen, and Petar Veli{\v{c}}kovi{\'c}.
\newblock Geometric deep learning: Grids, groups, graphs, geodesics, and gauges.
\newblock {\em arXiv preprint arXiv:2104.13478}, 2021.

\bibitem{chen2021geoqa}
Jiaqi Chen, Jianheng Tang, Jinghui Qin, Xiaodan Liang, Lingbo Liu, Eric~P. Xu, and Liang Lin.
\newblock {GeoQA}: A geometric question answering benchmark towards multimodal numerical reasoning.
\newblock In {\em Findings of the Association for Computational Linguistics (ACL-IJCNLP)}, pages 513--523, 2021.

\bibitem{cohen2016group}
Taco Cohen and Max Welling.
\newblock Group equivariant convolutional networks.
\newblock In {\em International conference on machine learning}, pages 2990--2999. PMLR, 2016.

\bibitem{cubuk2019autoaugment}
Ekin~D. Cubuk, Barret Zoph, Dandelion Man{\'e}, Vijay Vasudevan, and Quoc~V. Le.
\newblock {AutoAugment}: Learning augmentation strategies from data.
\newblock In {\em Proceedings of the IEEE/CVF Conference on Computer Vision and Pattern Recognition (CVPR)}, pages 113--123, 2019.

\bibitem{cubuk2020randaugment}
Ekin~D. Cubuk, Barret Zoph, Jonathon Shlens, and Quoc~V. Le.
\newblock {RandAugment}: Practical automated data augmentation with a reduced search space.
\newblock In {\em Proceedings of the IEEE/CVF Conference on Computer Vision and Pattern Recognition Workshops (CVPRW)}, pages 702--703, 2020.

\bibitem{fu2025chip}
Jinlan Fu, Shenzhen Huangfu, Hao Fei, Xiaoyu Shen, Bryan Hooi, Xipeng Qiu, and See-Kiong Ng.
\newblock {CHiP}: Cross-modal hierarchical direct preference optimization for multimodal {LLMs}.
\newblock In {\em International Conference on Learning Representations}, 2025.

\bibitem{geirhos2020shortcut}
Robert Geirhos, J{\"o}rn-Henrik Jacobsen, Claudio Michaelis, Richard Zemel, Wieland Brendel, Matthias Bethge, and Felix~A Wichmann.
\newblock Shortcut learning in deep neural networks.
\newblock {\em Nature Machine Intelligence}, 2(11):665--673, 2020.

\bibitem{guan2024hallusionbench}
Tianrui Guan, Fuxiao Liu, Xiyang Wu, Ruiqi Xian, Zongxia Li, Xiaoyu Liu, Xijun Wang, Lichang Chen, Furong Huang, Yaser Yacoob, Dinesh Manocha, and Tianyi Zhou.
\newblock {HallusionBench}: An advanced diagnostic suite for entangled language hallucination and visual illusion in large vision-language models.
\newblock In {\em Proceedings of the IEEE/CVF Conference on Computer Vision and Pattern Recognition}, 2024.

\bibitem{guo2025deepseekr1}
Daya Guo, Dejian Yang, Haowei Zhang, Junxiao Song, Ruoyu Zhang, Runxin Xu, Qihao Zhu, Shirong Ma, Peiyi Wang, Xiao Bi, et~al.
\newblock {DeepSeek-R1}: Incentivizing reasoning capability in {LLMs} via reinforcement learning.
\newblock {\em arXiv preprint arXiv:2501.12948}, 2025.

\bibitem{hendrycks2020augmix}
Dan Hendrycks, Norman Mu, Ekin~Dogus Cubuk, Barret Zoph, Justin Gilmer, and Balaji Lakshminarayanan.
\newblock Augmix: A simple data processing method to improve robustness and uncertainty.
\newblock In {\em International Conference on Learning Representations}, 2020.

\bibitem{hermann2020shapes}
Katherine Hermann, Ting Chen, and Simon Kornblith.
\newblock The origins and prevalence of texture bias in convolutional neural networks.
\newblock {\em Advances in neural information processing systems}, 33:19000--19015, 2020.

\bibitem{huang2025boosting}
Qihan Huang, Weilong Dai, Jinlong Liu, Wanggui He, Hao Jiang, Mingli Song, Jingyuan Chen, Chang Yao, and Jie Song.
\newblock Boosting mllm reasoning with text-debiased hint-grpo.
\newblock In {\em Proceedings of the IEEE/CVF International Conference on Computer Vision}, pages 4848--4857, 2025.

\bibitem{huang2026vppo}
Siyuan Huang, Xiaoye Qu, Yafu Li, Yun Luo, Zefeng He, Daizong Liu, and Yu~Cheng.
\newblock Spotlight on token perception for multimodal reinforcement learning.
\newblock In {\em International Conference on Learning Representations}, 2026.

\bibitem{huang2025visionr1}
Wenxuan Huang, Bohan Jia, Zijie Zhai, Shaosheng Cao, Zheyu Ye, Fei Zhao, Yao Hu, and Shaohui Lin.
\newblock {Vision-R1}: Incentivizing reasoning capability in multimodal large language models.
\newblock {\em arXiv preprint arXiv:2503.06749}, 2025.

\bibitem{hudson2019gqa}
Drew~A Hudson and Christopher~D Manning.
\newblock Gqa: A new dataset for real-world visual reasoning and compositional question answering.
\newblock In {\em Proceedings of the IEEE/CVF conference on computer vision and pattern recognition}, pages 6700--6709, 2019.

\bibitem{team2025kimi}
{Kimi Team}.
\newblock Kimi k1.5: Scaling reinforcement learning with {LLMs}.
\newblock {\em arXiv preprint arXiv:2501.12599}, 2025.

\bibitem{lecun1989}
Yann LeCun.
\newblock Generalization and network design strategies.
\newblock In R.~Pfeifer, Z.~Schreter, F.~Fogelman-Souli{\'e}, and L.~Steels, editors, {\em Connectionism in Perspective}. Elsevier, 1989.

\bibitem{li2024llavaonevision}
Bo~Li, Yuanhan Zhang, Dong Guo, Renrui Zhang, Feng Li, Hao Zhang, Kaichen Zhang, Peiyuan Zhang, Yanwei Li, Ziwei Liu, et~al.
\newblock Llava-onevision: Easy visual task transfer.
\newblock {\em arXiv preprint arXiv:2408.03326}, 2024.

\bibitem{li2023pope}
Yifan Li, Yifan Du, Kun Zhou, Jinpeng Wang, Wayne~Xin Zhao, and Ji-Rong Wen.
\newblock Evaluating object hallucination in large vision-language models.
\newblock {\em arXiv preprint arXiv:2305.10355}, 2023.

\bibitem{li2025eagle}
Zhiqi Li, Guo Chen, Shilong Liu, Shihao Wang, Vibashan VS, Yishen Ji, Shiyi Lan, Hao Zhang, Yilin Zhao, Subhashree Radhakrishnan, et~al.
\newblock Eagle 2: Building post-training data strategies from scratch for frontier vision-language models.
\newblock {\em arXiv preprint arXiv:2501.14818}, 2025.

\bibitem{liu2025noisyrollout}
Xiangyan Liu, Jinjie Ni, Zijian Wu, Chao Du, Longxu Dou, Haonan Wang, Tianyu Pang, and Michael~Qizhe Shieh.
\newblock Noisyrollout: Reinforcing visual reasoning with data augmentation.
\newblock {\em arXiv preprint arXiv:2504.13055}, 2025.

\bibitem{liu2025mmbench}
Yuan Liu, Haodong Duan, Yuanhan Zhang, Bo~Li, Songyang Zhang, Wangbo Zhao, Yike Yuan, Jiaqi Wang, Conghui He, Ziwei Liu, et~al.
\newblock Mmbench: Is your multi-modal model an all-around player?
\newblock In {\em European Conference on Computer Vision}, pages 216--233. Springer, 2025.

\bibitem{liu2025understanding}
Zichen Liu, Changyu Chen, Wenjun Li, Penghui Qi, Tianyu Pang, Chao Du, Wee~Sun Lee, and Min Lin.
\newblock Understanding r1-zero-like training: A critical perspective.
\newblock {\em arXiv preprint arXiv:2503.20783}, 2025.

\bibitem{liu2025visualreasoning}
Ziyu Liu, Zeyi Sun, Yuhang Zang, Xiaoyi Dong, Yuhang Cao, Haodong Duan, Dahua Lin, and Jiaqi Wang.
\newblock Visual-{RFT}: Visual reinforcement fine-tuning.
\newblock {\em arXiv preprint arXiv:2503.01785}, 2025.

\bibitem{lu2023mathvista}
Pan Lu, Hritik Bansal, Tony Xia, Jiacheng Liu, Chunyuan Li, Hannaneh Hajishirzi, Hao Cheng, Kai-Wei Chang, Michel Galley, and Jianfeng Gao.
\newblock {MathVista}: Evaluating mathematical reasoning of foundation models in visual contexts.
\newblock {\em arXiv preprint arXiv:2310.02255}, 2023.

\bibitem{meng2025mm}
Fanqing Meng, Lingxiao Du, Zongkai Liu, Zhixiang Zhou, Quanfeng Lu, Daocheng Fu, Botian Shi, Wenhai Wang, Junjun He, Kaipeng Zhang, Ping Luo, Yu~Qiao, Qiaosheng Zhang, and Wenqi Shao.
\newblock Mm-eureka: Exploring visual aha moment with rule-based large-scale reinforcement learning.
\newblock {\em arXiv preprint arXiv:2503.07365}, 2025.

\bibitem{miyato2018vat}
Takeru Miyato, Shin-ichi Maeda, Masanori Koyama, and Shin Ishii.
\newblock Virtual adversarial training: A regularization method for supervised and semi-supervised learning.
\newblock volume~41, pages 1979--1993, 2018.

\bibitem{ouyang2022instructgpt}
Long Ouyang, Jeffrey Wu, Xu~Jiang, Diogo Almeida, Carroll Wainwright, Pamela Mishkin, Chong Zhang, Sandhini Agarwal, Katarina Slama, Alex Ray, et~al.
\newblock Training language models to follow instructions with human feedback.
\newblock {\em Advances in neural information processing systems}, 35:27730--27744, 2022.

\bibitem{rafailov2023dpo}
Rafael Rafailov, Archit Sharma, Eric Mitchell, Stefano Ermon, Christopher~D. Manning, and Chelsea Finn.
\newblock Direct preference optimization: Your language model is secretly a reward model.
\newblock In {\em Advances in Neural Information Processing Systems}, 2023.

\bibitem{sagawa2020group}
Shiori Sagawa, Pang~Wei Koh, Tatsunori~B. Hashimoto, and Percy Liang.
\newblock Distributionally robust neural networks for group shifts: On the importance of regularization for worst-case generalization.
\newblock In {\em International Conference on Learning Representations (ICLR)}, 2020.

\bibitem{schulman2017ppo}
John Schulman, Filip Wolski, Prafulla Dhariwal, Alec Radford, and Oleg Klimov.
\newblock Proximal policy optimization algorithms.
\newblock In {\em arXiv preprint arXiv:1707.06347}, 2017.

\bibitem{shao2024deepseekmath}
Zhihong Shao, Peiyi Wang, Qihao Zhu, Runxin Xu, Junxiao Song, Xiao Bi, Haowei Zhang, Mingchuan Zhang, YK~Li, Yang Wu, et~al.
\newblock Deepseekmath: Pushing the limits of mathematical reasoning in open language models.
\newblock {\em arXiv preprint arXiv:2402.03300}, 2024.

\bibitem{shen2025vlmr1}
Haozhan Shen, Peng Liu, Jingcheng Li, Chunxin Fang, Yibo Ma, Jiajia Liao, Qiaoli Shen, Zilun Zhang, Kangjia Zhao, Qianqian Zhang, Ruochen Xu, and Tiancheng Zhao.
\newblock {VLM-R1}: A stable and generalizable {R1}-style large vision-language model.
\newblock {\em arXiv preprint arXiv:2504.07615}, 2025.

\bibitem{sun2023llavarlhf}
Zhiqing Sun, Sheng Shen, Shengcao Cao, Haotian Liu, Chunyuan Li, Yikang Shen, Chuang Gan, Liang-Yan Gui, Yu-Xiong Wang, Yiming Yang, Kurt Keutzer, and Trevor Darrell.
\newblock Aligning large multimodal models with factually augmented {RLHF}.
\newblock In {\em Findings of the Association for Computational Linguistics: ACL}, 2024.

\bibitem{tan2025reasonrft}
Huajie Tan, Yuheng Ji, Xiaoshuai Hao, Minglan Lin, Pengwei Wang, Zhongyuan Wang, and Shanghang Zhang.
\newblock {Reason-RFT}: Reinforcement fine-tuning for visual reasoning.
\newblock {\em arXiv preprint arXiv:2503.20752}, 2025.

\bibitem{tarvainen2017mean}
Antti Tarvainen and Harri Valpola.
\newblock Mean teachers are better role models: Weight-averaged consistency targets improve semi-supervised deep learning results.
\newblock In {\em Advances in Neural Information Processing Systems (NeurIPS)}, volume~30, 2017.

\bibitem{wang2024mdpo}
Fei Wang, Wenxuan Zhou, James~Y. Huang, Nan Xu, Sheng Zhang, Hoifung Poon, and Muhao Chen.
\newblock {mDPO}: Conditional preference optimization for multimodal large language models.
\newblock In {\em Proceedings of the Conference on Empirical Methods in Natural Language Processing}, 2024.

\bibitem{wang2023amber}
Junyang Wang, Yuhang Wang, Guohai Xu, Jing Zhang, Yukai Gu, Haitao Jia, Jiaqi Wang, Haiyang Xu, Ming Yan, Ji~Zhang, and Jitao Sang.
\newblock {AMBER}: An {LLM}-free multi-dimensional benchmark for {MLLMs} hallucination evaluation.
\newblock {\em arXiv preprint arXiv:2311.07397}, 2023.

\bibitem{wang2026papo}
Zhenhailong Wang, Xuehang Guo, Sofia Stoica, Haiyang Xu, Hongru Wang, Hyeonjeong Ha, Xiusi Chen, Yangyi Chen, Ming Yan, Fei Huang, and Heng Ji.
\newblock Perception-aware policy optimization for multimodal reasoning.
\newblock In {\em International Conference on Learning Representations}, 2026.

\bibitem{weiler2019general}
Maurice Weiler and Gabriele Cesa.
\newblock General {E(2)}-equivariant steerable {CNNs}.
\newblock In {\em Advances in Neural Information Processing Systems (NeurIPS)}, volume~32, 2019.

\bibitem{xiao2025perceptionr1}
Tong Xiao, Xin Xu, Zhenya Huang, Hongyu Gao, Quan Liu, Qi~Liu, and Enhong Chen.
\newblock {Perception-R1}: Advancing multimodal reasoning capabilities of {MLLMs} via visual perception reward.
\newblock {\em arXiv preprint arXiv:2506.07218}, 2025.

\bibitem{xie2020uda}
Qizhe Xie, Zihang Dai, Eduard Hovy, Minh-Thang Luong, and Quoc~V. Le.
\newblock Unsupervised data augmentation for consistency training.
\newblock In {\em Advances in Neural Information Processing Systems (NeurIPS)}, volume~33, pages 6256--6268, 2020.

\bibitem{xie2024vdpo}
Yuxi Xie, Guanzhen Li, Xiao Xu, and Min-Yen Kan.
\newblock {V-DPO}: Mitigating hallucination in large vision language models via vision-guided direct preference optimization.
\newblock In {\em Findings of the Association for Computational Linguistics: EMNLP}, 2024.

\bibitem{xu2025llava-cot}
Guowei Xu, Peng Jin, Ziang Wu, Hao Li, Yibing Song, Lichao Sun, and Li~Yuan.
\newblock Llava-cot: Let vision language models reason step-by-step.
\newblock In {\em Proceedings of the IEEE/CVF International Conference on Computer Vision}, pages 2087--2098, 2025.

\bibitem{yao2025sharegrpo}
Huanjin Yao, Qixiang Yin, Jingyi Zhang, Min Yang, Yibo Wang, Wenhao Wu, Fei Su, Li~Shen, Minghui Qiu, Dacheng Tao, and Jiaxing Huang.
\newblock {R1-ShareVL}: Incentivizing reasoning capability of multimodal large language models via share-{GRPO}.
\newblock {\em arXiv preprint arXiv:2505.16673}, 2025.

\bibitem{yin2024survey}
Shukang Yin, Chaoyou Fu, Sirui Zhao, Ke~Li, Xing Sun, Tong Xu, and Enhong Chen.
\newblock A survey on multimodal large language models.
\newblock {\em National Science Review}, 11(12):nwae403, 2024.

\bibitem{yu2024rlhfv}
Tianyu Yu, Yuan Yao, Haoye Zhang, Taiwen He, Yifeng Han, Ganqu Cui, Jinyi Hu, Zhiyuan Liu, Hai-Tao Zheng, Maosong Sun, and Tat-Seng Chua.
\newblock {RLHF-V}: Towards trustworthy {MLLMs} via behavior alignment from fine-grained correctional human feedback.
\newblock In {\em Proceedings of the IEEE/CVF Conference on Computer Vision and Pattern Recognition}, 2024.

\bibitem{zaheer2017deepsets}
Manzil Zaheer, Satwik Kottur, Siamak Ravanbakhsh, Barnabas Poczos, Russ~R Salakhutdinov, and Alexander~J Smola.
\newblock Deep sets.
\newblock {\em Advances in neural information processing systems}, 30, 2017.

\bibitem{zhang2025r1vl}
Jingyi Zhang, Jiaxing Huang, Huanjin Yao, Shunyu Liu, Xikun Zhang, Shijian Lu, and Dacheng Tao.
\newblock {R1-VL}: Learning to reason with multimodal large language models via step-wise group relative policy optimization.
\newblock {\em arXiv preprint arXiv:2503.12937}, 2025.

\bibitem{zhang2025lmms}
Kaichen Zhang, Bo~Li, Peiyuan Zhang, Fanyi Pu, Joshua~Adrian Cahyono, Kairui Hu, Shuai Liu, Yuanhan Zhang, Jingkang Yang, Chunyuan Li, et~al.
\newblock Lmms-eval: Reality check on the evaluation of large multimodal models.
\newblock In {\em Findings of the Association for Computational Linguistics: NAACL 2025}, pages 881--916, 2025.

\bibitem{zhang2025edge}
Xingjian Zhang, Siwei Wen, Wenjun Wu, and Lei Huang.
\newblock Edge-grpo: Entropy-driven grpo with guided error correction for advantage diversity.
\newblock {\em arXiv preprint arXiv:2507.21848}, 2025.

\bibitem{zheng2024reefknot}
Kening Zheng, Junkai Chen, Yibo Yan, Xin Zou, and Xuming Hu.
\newblock Reefknot: A comprehensive benchmark for relation hallucination evaluation, analysis and mitigation in multimodal large language models.
\newblock {\em arXiv preprint arXiv:2408.09429}, 2024.

\bibitem{zhou2025mitigating}
Guanyu Zhou, Yibo Yan, Xin Zou, Kun Wang, Aiwei Liu, and Xuming Hu.
\newblock Mitigating modality prior-induced hallucinations in multimodal large language models via deciphering attention causality.
\newblock In {\em International Conference on Learning Representations}, volume 2025, pages 54415--54439, 2025.

\bibitem{ziegler2019rlhf}
Daniel~M. Ziegler, Nisan Stiennon, Jeffrey Wu, Tom~B. Brown, Alec Radford, Dario Amodei, Paul Christiano, and Geoffrey Irving.
\newblock Fine-tuning language models from human preferences.
\newblock {\em arXiv preprint arXiv:1909.08593}, 2019.

\bibitem{zou2026learning}
Xin Zou, Ruimeng Liu, Chang Tang, Zhenglai Li, Xinwang Liu, Kunlun He, and Wanqing Li.
\newblock Learning disentangled representations for generalized multi-view clustering.
\newblock {\em IEEE Transactions on Pattern Analysis and Machine Intelligence}, 2026.

\bibitem{zou2023dpnet}
Xin Zou, Chang Tang, Xiao Zheng, Zhenglai Li, Xiao He, Shan An, and Xinwang Liu.
\newblock Dpnet: Dynamic poly-attention network for trustworthy multi-modal classification.
\newblock In {\em Proceedings of the 31st ACM international conference on multimedia}, pages 3550--3559, 2023.

\bibitem{zou2024look}
Xin Zou, Yizhou Wang, Yibo Yan, Sirui Huang, Kening Zheng, Junkai Chen, Chang Tang, and Xuming Hu.
\newblock Look twice before you answer: Memory-space visual retracing for hallucination mitigation in multimodal large language models.
\newblock {\em arXiv preprint arXiv:2410.03577}, 2024.

\end{thebibliography}
}

\newpage
\appendix
\hypersetup{linkcolor=black}
\etocdepthtag.toc{mtappendix}
\etocsettagdepth{mtchapter}{none}
\etocsettagdepth{mtappendix}{subsubsection}
\tableofcontents
\clearpage

{\large\textbf{$\infty$ Technical Appendices and Supplements}}

This appendix provides supplementary details for ConsistRoll, including benchmark descriptions, implementation settings, extended ablations, cross-view stability metrics, and theoretical proofs.

\section{Detailed Experiment Settings}
\label{apx:setting}

\subsection{Training Setup}
\label{apx:training}
We use Qwen2.5-VL-3B and Qwen2.5-VL-7B as the backbone MLLMs. All RLVR methods are trained on the same 8,082 math-and-logic samples extracted from LLaVA-CoT~\citep{xu2025llava-cot}. The response format is fixed as \texttt{<think>...</think><answer>...</answer>} so that the verifier can parse the final answer consistently. We train for one epoch with AdamW and a learning rate of $5\times10^{-5}$. The total rollout number is fixed to $n=4$ for all methods. For ConsistRoll, invariant samples allocate $n/2$ rollouts to the original view and $n/2$ rollouts to the transformed view, so the total rollout budget is identical to GRPO. Unless otherwise specified, the consistency weight is set to $\lambda=0.5$ and the transformation is image rotation.

\subsection{Benchmarks and Metrics}
\label{apx:dataset}
We evaluate mathematical and logical reasoning with MathVerse, MathVision, MathVista~\citep{lu2023mathvista}, LogicVista, MMLU-Pro Math, JMMMU-Math, MMStar, DynaMath, and JMMMU-Pro. These benchmarks cover diagram reasoning, spatial and geometric relations, symbolic calculation, and compositional visual reasoning. We report the benchmark-specific accuracy or score returned by the LMMs-Eval toolkit~\citep{zhang2025lmms}.

For general multimodal understanding, we use GQA~\citep{hudson2019gqa}, MMBench~\citep{liu2025mmbench}, MMStar, and LLaVA-in-the-Wild. These benchmarks test visual recognition, scene understanding, instruction following, and open-ended multimodal reasoning. For LLaVA-in-the-Wild, we report the overall score and its complex reasoning, conversation, and detail sub-scores.

For hallucination-sensitive evaluation, we use HallusionBench~\citep{guan2024hallusionbench}. We report answer-level accuracy (aAcc), figure-level accuracy (fAcc), question-pair accuracy (qAcc), and their average. These metrics test whether the model can avoid language-driven or visually implausible answers under controlled multimodal question pairs.

For cross-view stability, we additionally evaluate original-rotation pairs from GeoQA~\citep{chen2021geoqa}. This evaluation is not used for training; it measures whether the model preserves its correctness state under answer-preserving rotations.

\subsection{Baselines}
\label{apx:baseline}
We compare against the following baselines. \textbf{Base} denotes the original Qwen2.5-VL checkpoint without task-specific post-training. \textbf{SFT} denotes supervised fine-tuning on the same training data and response format. \textbf{GRPO} denotes standard group-relative policy optimization with verifiable answer rewards~\citep{shao2024deepseekmath}. \textbf{Hint-GRPO} denotes a hint-enhanced GRPO baseline~\citep{huang2025boosting}. ConsistRoll is applied to both GRPO and Hint-GRPO by modifying only the rollout grouping and reward construction.

\subsection{More Ablation}
In Table~\ref{tab:lambda-ablation}, \textit{w/ Consist-only} denotes a degenerate variant where the consistency bonus is applied solely based on answer agreement across views, without checking whether either answer is correct.
As shown in Table~\ref{tab:lambda-ablation}, this variant brings limited or even negative gains compared with the full ConsistRoll objective.
This result is consistent with our theoretical analysis: answer agreement alone is not a reliable learning signal, since two transformed views can produce the same incorrect answer and still receive a positive consistency reward.
In contrast, ConsistRoll gates the consistency bonus by correctness, so the cross-view signal reinforces invariant reasoning only when both views lead to the verified answer.
The gap between \textit{w/ Consist-only} and ConsistRoll therefore supports the necessity of the correctness-aware consistency design.

\section{Additional Related Work}
\label{apx:related}

\paragraph{RLVR and GRPO for multimodal reasoning.}
Reinforcement learning has become a central post-training tool for large language models, from RLHF with learned reward models~\citep{ziegler2019rlhf,ouyang2022instructgpt} to reinforcement learning with verifiable rewards (RLVR), where rule-based answer checkers provide scalable supervision. GRPO~\citep{shao2024deepseekmath} removes the critic by normalizing rewards within a sampled group, and has been widely adopted by recent reasoning models such as DeepSeek-R1~\citep{guo2025deepseekr1} and Kimi k1.5~\citep{team2025kimi}. This paradigm has quickly moved into the multimodal domain. Visual-RFT~\citep{liu2025visualreasoning}, Vision-R1~\citep{huang2025visionr1}, Reason-RFT~\citep{tan2025reasonrft}, R1-VL~\citep{zhang2025r1vl}, and VLM-R1~\citep{shen2025vlmr1} show that GRPO-style RL can improve visual perception, mathematical reasoning, and out-of-domain generalization of MLLMs. More recent variants refine the reward or advantage estimator, including step-wise reasoning rewards, visual perception rewards, and reward sharing across transformed question variants~\citep{zhang2025r1vl,xiao2025perceptionr1,yao2025sharegrpo}. ConsistRoll is complementary to this line: rather than designing denser rewards, it changes the comparison unit of RLVR by placing answer-preserving visual views in the same group.

\paragraph{Multimodal preference optimization and hallucination alignment.}
Another line of work improves MLLM reliability through preference learning \citep{zou2024look,zou2026learning,zheng2024reefknot,zhou2025mitigating}. LLaVA-RLHF adapts RLHF to large multimodal models with factually augmented reward modeling and introduces MMHal-Bench for hallucination evaluation~\citep{sun2023llavarlhf}. RLHF-V uses fine-grained correctional human feedback to reduce hallucinated content~\citep{yu2024rlhfv}. DPO~\citep{rafailov2023dpo} has also been extended to multimodal settings, including mDPO, V-DPO, and CHiP~\citep{wang2024mdpo,xie2024vdpo,fu2025chip}. These methods typically require human, model-generated, or synthetically constructed preference pairs. ConsistRoll targets a different regime: tasks with verifiable answers but no preference annotations. Its preference-like signal is generated online from a correctness-aware cross-view predicate.

\paragraph{Data augmentation, consistency regularization, and invariance.}
Data augmentation is a standard way to inject invariances into vision systems. Classical and automated policies such as AutoAugment, RandAugment, and AugMix improve robustness by exposing models to label-preserving transformations~\citep{cubuk2019autoaugment,cubuk2020randaugment,hendrycks2020augmix,zou2023dpnet}. Prediction consistency across perturbations has also been widely studied in semi-supervised and robust learning, including Mean Teacher, Virtual Adversarial Training, and UDA~\citep{tarvainen2017mean,miyato2018vat,xie2020uda}. These approaches add probability-matching losses, whereas ConsistRoll operates in a language-generation RL setting where answers are sampled, parsed, and scored by verifiable rewards. This distinction matters because consistency alone can be maximized by repeating the same wrong answer, while correctness alone reduces to pointwise RLVR.

\begin{table*}[t]
\centering
\renewcommand{\arraystretch}{1.15}
\setlength{\tabcolsep}{4.2pt}
\caption{Method-level comparison with representative RLVR-based MLLM reasoning methods. Most prior methods improve the reward design, data construction, or exploration space. ConsistRoll instead changes the RL comparison unit by coupling semantically invariant visual views and applying a correctness-aware consistency reward, while keeping the total rollout budget unchanged.}
\label{tab:method-feature-comparison}
\resizebox{\linewidth}{!}{
\begin{tabular}{l c c c c c c >{\centering\arraybackslash}p{2.9cm}}
\toprule[1.15pt]
\textbf{Method}
& \makecell[c]{\textbf{RLVR}\\\textbf{Training}}
& \makecell[c]{\textbf{Variant}\\\textbf{Source}}
& \makecell[c]{\textbf{Cross-Variant}\\\textbf{Credit}}
& \makecell[c]{\textbf{Answer-Preserving}\\\textbf{Visual Views}}
& \makecell[c]{\textbf{Correctness-Aware}\\\textbf{Consistency}}
& \makecell[c]{\textbf{No Extra}\\\textbf{Rollouts}}
& \makecell[c]{\textbf{Modified}\\\textbf{Component(s)}} \\
\midrule
GRPO~\citep{shao2024deepseekmath}
& \cmark & -- & \xmark & \xmark & \xmark & \cmark & group-normalized outcome reward \\
Visual-RFT~\citep{liu2025visualreasoning}
& \cmark & task verifier & \xmark & \xmark & \xmark & \cmark & task-specific verifiable rewards \\
Vision-R1~\citep{huang2025visionr1}
& \cmark & cold-start CoT & \xmark & \xmark & \xmark & \cmark & cold-start data and training schedule \\
Reason-RFT~\citep{tan2025reasonrft}
& \cmark & curated CoT data & \xmark & \xmark & \xmark & \cmark & two-stage SFT--RL training \\
R1-VL~\citep{zhang2025r1vl}
& \cmark & reasoning steps & \xmark & \xmark & \xmark & \cmark & step-wise reward design \\
VLM-R1~\citep{shen2025vlmr1}
& \cmark & task ground truth & \xmark & \xmark & \xmark & \cmark & rule-based rewards and data filtering \\
Perception-R1~\citep{xiao2025perceptionr1}
& \cmark & visual annotations & \xmark & \xmark & \xmark & \cmark & perception reward with LLM judge \\
NoisyRollout~\citep{liu2025noisyrollout}
& \cmark & clean/noisy images & \textcolor{orange!85!black}{$\triangle$} & \textcolor{orange!85!black}{$\triangle$} & \xmark & \cmark & rollout-level data augmentation \\
\rowcolor{lightgray}
 Share-GRPO~\citep{yao2025sharegrpo}
& \cellcolor{lightgray}\cmark & \cellcolor{lightgray}question variants & \cellcolor{lightgray}\cmark & \cellcolor{lightgray}\xmark & \cellcolor{lightgray}\xmark & \cellcolor{lightgray}\textcolor{orange!85!black}{$\triangle$} & \cellcolor{lightgray}trajectory and reward sharing \\
Hint-GRPO~\citep{huang2025boosting}
& \cmark & adaptive hints & \xmark & \xmark & \xmark & \textcolor{orange!85!black}{$\triangle$} & hint-based reward and text-bias calibration \\
\rowcolor{lightgray}
\textbf{ConsistRoll (ours)}
& \cellcolor{lightgray}\cmark & \cellcolor{lightgray}\makecell[c]{original / transformed\\visual views} & \cellcolor{lightgray}\cmark & \cellcolor{lightgray}\cmark & \cellcolor{lightgray}\cmark & \cellcolor{lightgray}\cmark & \cellcolor{lightgray}rollout grouping and consistent rewarding \\
\bottomrule[1.15pt]
\end{tabular}
}
\vspace{0.25em}
\begin{flushleft}
\small
\textbf{Note:} Where \cmark\; indicates that the property is explicitly used, \xmark\; indicates that it is not used, and \textcolor{orange!85!black}{$\triangle$}  indicates a partial or indirect use. ``No extra rollouts'' is marked relative to a fixed GRPO-style rollout budget, ConsistRoll keeps the total number of generated responses unchanged by allocating half of the rollouts to each view.
\end{flushleft}  \vspace{-2em}
\end{table*}

\paragraph{Inductive biases from symmetry and invariance.}
The broader motivation of ConsistRoll follows the principle that learning systems generalize better when their inductive biases match task symmetries. Convolutional networks encode translation locality and weight sharing~\citep{lecun1989}; group-equivariant CNNs and steerable networks build in rotation/reflection equivariance~\citep{cohen2016group,weiler2019general}; Deep Sets encode permutation invariance~\citep{zaheer2017deepsets}; and geometric deep learning provides a general framework for symmetry-aware architectures~\citep{bronstein2021geometric}. These works encode invariance structurally. ConsistRoll encodes it algorithmically through the RL objective, which is useful for modern MLLMs where changing the architecture or pre-training recipe is expensive.

\paragraph{Shortcut learning and spurious visual cues.}
Vision models often exploit texture, background, or dataset-specific cues rather than the intended semantic rule~\citep{geirhos2020shortcut,hermann2020shapes}. Group DRO and related methods mitigate spurious correlations when group annotations or environments are available~\citep{sagawa2020group}. In RLVR-trained MLLMs, the analogous failure is reward-level: a policy can be rewarded for correct answers on canonical views while relying on view-specific shortcuts that fail under equivalent transformations. ConsistRoll addresses this failure without group labels by constructing the comparison online.

\subsection{Method Feature Comparison}
\label{apx:method-feature-comparison}

Table~\ref{tab:method-feature-comparison} summarizes the key design differences between ConsistRoll and representative RLVR-based MLLM reasoning methods. Existing approaches mainly improve multimodal RL through stronger verifiable rewards, cold-start reasoning data, step-wise reasoning supervision, perception rewards, hint-based exploration, or variant sharing over an expanded question space. By contrast, ConsistRoll targets a different failure mode: standard RLVR can reward each view independently and thus does not penalize view-dependent shortcuts. ConsistRoll addresses this issue by placing semantically invariant visual views in the same rollout group and assigning a joint reward only when paired completions are both correct and answer-consistent. This design turns answer-preserving transformations into an explicit cross-view credit-assignment signal without changing the backbone model or increasing the total rollout budget.

\section{More Implementation}
\subsection{Construction of the Semantically-Invariant Training Subset}
\label{app:invariant_subset}

Algorithm~\ref{alg:consistroll} applies the view-coupled reward only to samples whose answers are expected to remain unchanged under the adopted visual transformations. 
Starting from the 8,082 training samples, we construct this semantically-invariant subset using a rule-based filtering procedure over the question text. 
Specifically, we remove VQA instances that contain view-dependent spatial expressions, including absolute or relative orientation cues such as left/right, upper/lower, top/bottom, front/behind, east/west/south/north, clockwise/counterclockwise, and other terms indicating position, direction, or ordering in the image. 
These questions may change their correct answers after geometric transformations and are therefore not suitable for consistency-based reward coupling.

After filtering, we obtain approximately 7,000 samples as the semantically-invariant training subset. 
This filtering step uses only the question text and does not require additional human annotations. 
For samples outside this subset, we keep the standard GRPO-style correctness reward and do not apply the cross-view consistency bonus, avoiding incorrect consistency constraints on transformation-sensitive questions.

\section{More Results}
\subsection{Full Hyperparameter Ablation Results}
\label{apx:lambda-full}

We provide the full benchmark-level results for the consistency weight $\lambda$ in Tables~\ref{tab:lambda-full-math} and~\ref{tab:lambda-full-general}. These tables complement the compact main-paper ablation by reporting every metric in the source spreadsheet while following the same style as Tables~\ref{tab1:main} and~\ref{tab2:main}.

\begin{table}[t]
    \centering
    \setlength{\tabcolsep}{3.0pt}
    \footnotesize
    \caption{Full $\lambda$ ablation results on math and reasoning benchmarks from ConsistRoll. All results are reported on Qwen2.5-VL-7B.}
    \vspace{0.25em}
    \label{tab:lambda-full-math}
    \resizebox{\linewidth}{!}{
    \begin{tabular}{l *{10}{>{\centering\arraybackslash}p{3.5em}}}
        \toprule[1.25pt]
        \multirow{2}{*}{\textbf{Setting}} & \multirow{2}{*}{\makecell[c]{\textbf{Math-}\\\textbf{Verse}}} & \multirow{2}{*}{\makecell[c]{\textbf{Math-}\\\textbf{Vision}}} & \multirow{2}{*}{\makecell[c]{\textbf{Math-}\\\textbf{Vista}}} & \multirow{2}{*}{\makecell[c]{\textbf{MMLU-}\\\textbf{ProMath}}} & \multirow{2}{*}{\makecell[c]{\textbf{JM$^{3}$U}\\\textbf{-Math}}} & \multirow{2}{*}{\makecell[c]{\textbf{Dyna-}\\\textbf{Math}}} & \multirow{2}{*}{\makecell[c]{\textbf{Logic-}\\\textbf{Vista}}} & \multicolumn{2}{c}{\textbf{MMStar}} & \multirow{2}{*}{\makecell[c]{\textbf{SciBench}\\\textbf{Math}}} \\
        \cmidrule[0.5pt](lr){9-10}
        & & & & & & & & \textbf{Logical} & \textbf{Math} & \\
        \midrule
        Qwen2.5-VL-7B & 52.41 & 26.64 & 68.60 & 51.00 & 36.67 & 47.86 & 41.96 & 60.97 & 66.31 & 45.58 \\
        GRPO & 53.58 & 27.63 & 68.00 & 52.48 & 33.33 & 51.92 & 41.96 & 60.80 & 65.47 & 49.66 \\
        $\lambda=0.0$ & 52.10 & 27.63 & 68.10 & 53.11 & 30.00 & 52.08 & 42.41 & 64.22 & 66.22 & 51.70 \\
        \cc ConsistRoll \scriptsize{(Ours)} & \cc 54.25 & \cc 28.89 & \cc 69.20 & \cc 53.22 & \cc 40.00 & \cc 52.61 & \cc 45.76 & \cc 63.37 & \cc 67.78 & \cc 48.30 \\
        \midrule
        $\lambda=0.1$ & 52.17 & 28.95 & 68.10 & 52.11 & 36.67 & 50.88 & 44.64 & 62.51 & 67.80 & 51.02 \\
        $\lambda=0.3$ & 51.56 & 29.93 & 68.80 & 54.33 & 33.33 & 51.88 & 44.20 & 62.82 & 66.41 & 46.94 \\
        $\lambda=0.5$ & 54.25 & 28.89 & 69.20 & 53.22 & 40.00 & 52.61 & 45.76 & 63.37 & 67.78 & 48.30 \\
        $\lambda=0.7$ & 50.37 & 28.95 & 69.00 & 52.18 & 36.67 & 51.76 & 45.54 & 63.62 & 65.26 & 48.30 \\
        $\lambda=0.8$ & 50.98 & 28.62 & 69.10 & 50.85 & 36.67 & 50.90 & 42.86 & 61.88 & 67.00 & 46.94 \\
        $\lambda=1.0$ & 51.36 & 30.92 & 67.90 & 52.48 & 33.33 & 52.02 & 46.21 & 61.66 & 66.14 & 45.58 \\
        Consistency-only & 50.52 & 26.97 & 68.40 & 51.96 & 33.33 & 50.70 & 43.75 & 62.51 & 64.51 & 46.26 \\
        \bottomrule[1.25pt]
    \end{tabular}
    }
    \vspace{-0.8em}
\end{table}

\begin{table}[t]
    \centering
    \setlength{\tabcolsep}{3.0pt}
    \footnotesize
    \caption{Full $\lambda$ ablation results on general and hallucination benchmarks from ConsistRoll. All results are reported on Qwen2.5-VL-7B. A-Acc, F-Acc, and Q-Acc denote the answer-level, figure-level, and question-pair accuracies of HallusionBench.}
    \vspace{0.25em}
    \label{tab:lambda-full-general}
    \resizebox{\linewidth}{!}{
    \begin{tabular}{l *{12}{>{\centering\arraybackslash}p{3.0em}}}
        \toprule[1.25pt]
        \multirow{2}{*}{\textbf{Setting}} & \multirow{2}{*}{\textbf{GQA}} & \multirow{2}{*}{\makecell[c]{\textbf{MMB}\\\textbf{Acc.}}} & \multirow{2}{*}{\makecell[c]{\textbf{MMStar}\\\textbf{Avg.}}} & \multicolumn{4}{c}{\textbf{LLaVA in the Wild}} & \multicolumn{2}{c}{\textbf{JMMMU-Pro}} & \multicolumn{3}{c}{\textbf{HallusionBench}} \\
        \cmidrule[0.5pt](lr){5-8} \cmidrule[0.5pt](lr){9-10} \cmidrule[0.5pt](lr){11-13}
        & & & & \textbf{All} & \textbf{Complex} & \textbf{Conv} & \textbf{Detail} & \textbf{Standard} & \textbf{Vision} & \textbf{A-Acc} & \textbf{F-Acc} & \textbf{Q-Acc} \\
        \midrule
        Qwen2.5-VL-7B & 60.81 & 82.82 & 62.09 & 75.80 & 85.90 & 83.00 & 51.90 & 46.06 & 45.15 & 62.67 & 36.42 & 36.04 \\
        GRPO & 61.04 & 83.16 & 61.45 & 68.30 & 72.70 & 78.80 & 49.30 & 46.29 & 46.44 & 63.41 & 38.44 & 37.14 \\
        $\lambda=0.0$ & 60.88 & 82.82 & 62.14 & 69.70 & 75.00 & 79.20 & 49.60 & 43.56 & 43.86 & 64.04 & 38.73 & 38.02 \\
        \cc ConsistRoll \scriptsize{(Ours)} & \cc 60.83 & \cc 83.16 & \cc 63.24 & \cc 76.10 & \cc 82.40 & \cc 85.60 & \cc 54.90 & \cc 46.67 & \cc 46.44 & \cc 63.93 & \cc 39.60 & \cc 38.24 \\
        \midrule
        $\lambda=0.1$ & 60.88 & 83.76 & 63.33 & 74.70 & 76.80 & 78.10 & 51.50 & 46.36 & 44.24 & 64.46 & 39.31 & 39.12 \\
        $\lambda=0.3$ & 60.76 & 82.91 & 62.47 & 75.40 & 86.30 & 81.80 & 51.10 & 45.61 & 45.38 & 63.20 & 37.86 & 37.36 \\
        $\lambda=0.5$ & 60.83 & 83.16 & 63.24 & 76.10 & 82.40 & 85.60 & 54.90 & 46.67 & 46.44 & 63.93 & 39.60 & 38.24 \\
        $\lambda=0.7$ & 60.94 & 83.33 & 62.43 & 72.30 & 75.50 & 85.70 & 51.10 & 46.82 & 46.89 & 63.51 & 38.15 & 38.24 \\
        $\lambda=0.8$ & 60.81 & 83.16 & 62.57 & 71.10 & 79.80 & 74.60 & 53.00 & 47.65 & 46.74 & 64.77 & 39.02 & 38.90 \\
        $\lambda=1.0$ & 61.11 & 82.99 & 62.60 & 71.40 & 75.20 & 81.40 & 53.30 & 46.59 & 47.12 & 64.04 & 39.02 & 37.58 \\
        Consistency-only & 60.87 & 82.73 & 62.77 & 69.80 & 77.40 & 75.80 & 47.40 & 45.76 & 46.06 & 63.93 & 39.60 & 38.24 \\
        \bottomrule[1.25pt]
    \end{tabular}
    }
    \vspace{-0.8em}
\end{table}

\subsection{Cross-View Stability Metrics}
\label{app:cross-view-stability}

We provide the detailed definitions of the cross-view stability metrics used in Table~\ref{tab:cross-view-stability}. 
For each original-rotation pair, let $o_i^{(0)}$ and $o_i^{(1)}$ denote the responses to the original and rotated views, respectively. 
Their correctness indicators are defined as $c_i^{(0)}=\mathbf{1}[\hat a(o_i^{(0)})=y_i^*]$ and $c_i^{(1)}=\mathbf{1}[\hat a(o_i^{(1)})=y_i^*]$. 
Given $N$ paired samples, we compute rotation-induced degradation, rotation-induced improvement, total flips, and correctness-state consistency as
\begin{equation}
\label{eq:cross-view-stability-app}
\begin{aligned}
D_{\mathrm{rot}}
&=\sum_{i=1}^{N} c_i^{(0)}(1-c_i^{(1)}),
\qquad
I_{\mathrm{rot}}
=\sum_{i=1}^{N} (1-c_i^{(0)})c_i^{(1)}, \\
F_{\mathrm{rot}}
&=D_{\mathrm{rot}}+I_{\mathrm{rot}},
\qquad
C_{\mathrm{rot}}
=1-\frac{F_{\mathrm{rot}}}{N}.
\end{aligned}
\end{equation}
Here $D_{\mathrm{rot}}$ counts cases where the model is correct on the original image but becomes wrong after rotation, while $I_{\mathrm{rot}}$ counts the opposite direction. 
Their sum $F_{\mathrm{rot}}$ measures the number of cross-view correctness-state changes, and $C_{\mathrm{rot}}$ measures the fraction of samples whose correctness state remains stable across the two views.

The results show that ConsistRoll substantially reduces rotation-induced instability. 
Compared with GRPO, ConsistRoll decreases $D_{\mathrm{rot}}$ from $65$ to $43$ and $F_{\mathrm{rot}}$ from $120$ to $94$, yielding a $3.45$-point improvement in $C_{\mathrm{rot}}$. 
The larger reduction in degradation than improvement is particularly important: ConsistRoll is better at preserving correct predictions when the visual input undergoes an answer-preserving transformation.

\section{Limitations}
\label{apx:limitations}

\paragraph{Dependence on valid answer-preserving transformations.}
ConsistRoll assumes that the transformed view preserves the ground-truth answer. This assumption is natural for transformations such as moderate rotations in many geometry or object-centric reasoning tasks, but it is not universal. Orientation-sensitive tasks, OCR-heavy images, fine-grained counting, spatial-language questions, or questions involving ``left'', ``right'', ``top'', and ``bottom'' may change their semantics after transformation. In such cases, applying a cross-view consistency reward can introduce incorrect supervision. In practice, ConsistRoll should therefore be used with transformation families that are either theoretically answer-preserving for the task or filtered by task-specific validity checks.

\paragraph{Verifier and parser dependence.}
Like other RLVR methods, ConsistRoll relies on a rule-based verifier to determine whether a generated answer is correct. Parser errors, ambiguous answer formats, or noisy ground-truth annotations can directly affect both the correctness reward and the correctness-gated consistency bonus. Although the response format reduces parsing ambiguity, the method does not solve the broader problem of open-ended multimodal verification. Extending ConsistRoll to free-form reasoning, long-form answers, or tasks without exact answer checkers would require stronger semantic verifiers.

\paragraph{Computation and rollout allocation.}
ConsistRoll keeps the total rollout budget identical to GRPO by allocating half of the rollouts to the original view and half to the transformed view. This design avoids additional generation cost, but it also reduces the number of samples per individual view compared with a same-view GRPO group of equal size.

\section{Broader Impacts}
\label{apx:broader-impacts}

ConsistRoll is a general training strategy for improving the stability of multimodal reasoning under answer-preserving visual transformations. Its main positive impact is to make MLLMs less sensitive to superficial visual changes, which can improve reliability in applications involving diagrams, educational content, scientific figures, geometric reasoning, and hallucination-sensitive visual question answering. Because the method does not require extra human annotations or architectural changes, it may also make robustness-oriented post-training more accessible to research groups with limited annotation budgets.

At the same time, more robust multimodal reasoning systems can be misused or over-trusted. In high-stakes contexts such as medical image analysis, legal evidence review, surveillance, or safety-critical robotics, improved cross-view stability should not be interpreted as sufficient evidence of correctness. Stable predictions can still be wrong when the task falls outside the verifier, the transformation changes the semantics, or the input distribution differs from the evaluated benchmarks. Deployment in such settings would require domain-specific validation, uncertainty estimation, human oversight, and compliance with relevant privacy and safety standards.


\section{Proofs and Supporting Analysis}
\label{app:proofs}

This appendix provides the detailed proof that cross-view consistency is a valid and learnable inductive bias for ConsistRoll. 
We first prove Proposition~\ref{prop:view-coupled-policy-improvement}. 
We then give two supporting results: a distributional compatibility lemma and a degeneracy analysis for consistency-only rewards. 
Finally, we connect the population objective to the finite-sample GRPO advantage used in the algorithm.

\subsection{Proof of Proposition~\ref{prop:view-coupled-policy-improvement}}
\label{app:proof-view-coupled}

The proof has three parts. 
First, we show \emph{validity}: under semantically-invariant transformations, the correctness-gated consistency objective is aligned with the invariant target and does not reward wrong-but-consistent predictions. 
Second, we show \emph{learnability}: the bias induces a concrete policy-gradient term that can be estimated from paired rollouts. 
Third, we prove the standard convergence-to-stationarity statement for stochastic gradient ascent on the induced population objective.

\paragraph{Step 1: notation and population objective.}
For a semantically-invariant pair, write
\begin{equation}
    x_0=(I,q),
    \qquad
    x_1=(T_\alpha(I),q),
    \qquad
    \pi_v(o)=\pi_\theta(o\mid x_v),
    \quad v\in\{0,1\}.
\end{equation}
Let $\mathcal O$ be the response space. 
The verifiable correctness indicator and the correctness-gated consistency bonus are
\begin{equation}
\label{eq:app-indicators}
    c_v(o)=\mathbbm{1}[\hat a(o)=y^*],
    \qquad
    b(o_0,o_1)=c_0(o_0)c_1(o_1)\mathbbm{1}[\kappa(o_0,o_1)=1].
\end{equation}
Because the verifier and parser are fixed, $c_v$ and $b$ do not directly depend on $\theta$. 
The unclipped ConsistRoll population objective is therefore
\begin{equation}
\label{eq:app-cr-objective-expanded}
\begin{aligned}
    \mathcal J_{\mathrm{CR}}(\theta)
    &=
    \mathbb E_{o_0\sim\pi_0,\,o_1\sim\pi_1}
    \left[
    \frac{1}{2}\big(c_0(o_0)+c_1(o_1)\big)+\lambda b(o_0,o_1)
    \right] \\
    &=
    \sum_{o_0\in\mathcal O}\sum_{o_1\in\mathcal O}
    \pi_0(o_0)\pi_1(o_1)
    \left[
    \frac{1}{2}\big(c_0(o_0)+c_1(o_1)\big)+\lambda b(o_0,o_1)
    \right] .
\end{aligned}
\end{equation}
This is the population counterpart of the reward used inside the clipped GRPO surrogate. 
As usual in policy-gradient analysis, we assume that differentiation can be exchanged with the finite or countable summation over responses.

\paragraph{Step 2: validity from semantic invariance.}
By Assumption~\ref{asm:semantic-invariance}, both views share the same ground-truth answer $y^*$. 
Therefore, the Bayes-correct answer is invariant across the pair: the desired behavior is not merely high accuracy on one view, but high accuracy on both views with the same parsed answer. 
This is the target that the consistency bias should favor.

Assume that the parser canonicalizes correct answers. 
Then, whenever $c_0(o_0)=1$ and $c_1(o_1)=1$, both completions parse to $y^*$, and hence $\kappa(o_0,o_1)=1$. 
Thus the correctness-gated bonus reduces to
\begin{equation}
\label{eq:app-bonus-reduces-product}
    b(o_0,o_1)=c_0(o_0)c_1(o_1).
\end{equation}
Define the view-level correctness probabilities
\begin{equation}
\label{eq:app-pv-def}
    p_v(\theta)
    :=
    \Pr_{o\sim\pi_\theta(\cdot\mid x_v)}[\hat a(o)=y^*]
    =
    \mathbb E_{o\sim\pi_\theta(\cdot\mid x_v)}[c_v(o)],
    \qquad v\in\{0,1\}.
\end{equation}
Since $o_0$ and $o_1$ are sampled independently conditioned on their views, Eq.~\eqref{eq:app-bonus-reduces-product} gives
\begin{equation}
\label{eq:app-bonus-factorization}
\begin{aligned}
    \mathbb E[b(o_0,o_1)]
    &=
    \mathbb E[c_0(o_0)c_1(o_1)] \\
    &=
    \left(\mathbb E_{o_0\sim\pi_\theta(\cdot\mid x_0)}[c_0(o_0)]\right)
    \left(\mathbb E_{o_1\sim\pi_\theta(\cdot\mid x_1)}[c_1(o_1)]\right) \\
    &=
    p_0(\theta)p_1(\theta).
\end{aligned}
\end{equation}
Substituting Eq.~\eqref{eq:app-pv-def} and Eq.~\eqref{eq:app-bonus-factorization} into Eq.~\eqref{eq:app-cr-objective-expanded} yields
\begin{equation}
\label{eq:app-prob-objective}
    \mathcal J_{\mathrm{CR}}(\theta)
    =
    \frac{1}{2}\big(p_0(\theta)+p_1(\theta)\big)
    +
    \lambda p_0(\theta)p_1(\theta).
\end{equation}
As a function of $(p_0,p_1)\in[0,1]^2$, this objective satisfies
\begin{equation}
\label{eq:app-prob-monotone}
    \frac{\partial \mathcal J_{\mathrm{CR}}}{\partial p_0}
    =
    \frac{1}{2}+\lambda p_1
    >0,
    \qquad
    \frac{\partial \mathcal J_{\mathrm{CR}}}{\partial p_1}
    =
    \frac{1}{2}+\lambda p_0
    >0,
    \qquad \lambda\geq 0 .
\end{equation}
Therefore, the population objective is strictly increasing in the correctness probability of each view. 
If the policy class can realize $p_0=p_1=1$, the global maximizer of the probability-level objective is attained at $p_0=p_1=1$. 
This proves alignment with the invariant target: the best solution is correct on both semantically equivalent views.

The same calculation also rules out the main invalid-bias failure mode. 
A policy that is perfectly consistent but always emits an incorrect canonical answer $a_0\neq y^*$ has $p_0=p_1=0$. 
By Eq.~\eqref{eq:app-prob-objective}, it obtains
\begin{equation}
    \mathcal J_{\mathrm{CR}}=0,
\end{equation}
whereas any policy with nonzero correctness on either view obtains a strictly positive correctness reward. 
Thus, the ConsistRoll bonus does not make wrong-but-consistent predictions optimal. 
The consistency signal is valid precisely because it is gated by dual correctness.

\paragraph{Step 3: learnability via the policy-gradient identity.}
We now show that the valid bias is also learnable: it appears as a differentiable policy-gradient signal and can be estimated from paired rollouts.
For one view $v$, the expected correctness reward is
\begin{equation}
    C_v(\theta)
    :=
    \mathbb E_{o\sim\pi_\theta(\cdot\mid x_v)}[c_v(o)]
    =
    \sum_{o\in\mathcal O}c_v(o)\pi_\theta(o\mid x_v).
\end{equation}
Taking the gradient gives
\begin{equation}
\label{eq:app-single-view-grad-steps}
\begin{aligned}
    \nabla_\theta C_v(\theta)
    &=
    \sum_{o\in\mathcal O}c_v(o)\nabla_\theta\pi_\theta(o\mid x_v)  \\
    &=
    \sum_{o\in\mathcal O}c_v(o)\pi_\theta(o\mid x_v)
    \nabla_\theta\log\pi_\theta(o\mid x_v) \\
    &=
    \mathbb E_{o\sim\pi_\theta(\cdot\mid x_v)}
    \left[
    c_v(o)\nabla_\theta\log\pi_\theta(o\mid x_v)
    \right].
\end{aligned}
\end{equation}
The second equality uses the score-function identity
$\nabla_\theta\pi_\theta(o\mid x_v)=\pi_\theta(o\mid x_v)\nabla_\theta\log\pi_\theta(o\mid x_v)$.

For the paired bonus, define
\begin{equation}
    B(\theta)
    :=
    \mathbb E_{o_0\sim\pi_0,\,o_1\sim\pi_1}[b(o_0,o_1)]
    =
    \sum_{o_0,o_1} b(o_0,o_1)\pi_0(o_0)\pi_1(o_1).
\end{equation}
Applying the product rule gives
\begin{equation}
\label{eq:app-bonus-grad-product}
\begin{aligned}
    \nabla_\theta B(\theta)
    &=
    \sum_{o_0,o_1} b(o_0,o_1)
    \left[
    \nabla_\theta\pi_0(o_0)\pi_1(o_1)
    +
    \pi_0(o_0)\nabla_\theta\pi_1(o_1)
    \right] \\
    &=
    \sum_{o_0,o_1} b(o_0,o_1)\pi_0(o_0)\pi_1(o_1)
    \left[
    \nabla_\theta\log\pi_0(o_0)
    +
    \nabla_\theta\log\pi_1(o_1)
    \right] \\
    &=
    \mathbb E_{o_0,o_1\sim\pi_\theta}
    \left[
    b(o_0,o_1)
    \left(
    \nabla_\theta\log\pi_\theta(o_0\mid x_0)
    +
    \nabla_\theta\log\pi_\theta(o_1\mid x_1)
    \right)
    \right].
\end{aligned}
\end{equation}
Combining Eq.~\eqref{eq:app-single-view-grad-steps} for $v=0,1$ with Eq.~\eqref{eq:app-bonus-grad-product} yields
\begin{equation}
\label{eq:app-full-gradient-decomposition}
\begin{aligned}
\nabla_\theta \mathcal J_{\mathrm{CR}}(\theta)
&=
\frac{1}{2}
\sum_{v=0}^{1}
\mathbb E_{o_v\sim\pi_\theta(\cdot\mid x_v)}
\left[
    c_v(o_v)\nabla_\theta\log\pi_\theta(o_v\mid x_v)
\right] \\
&\quad+
\lambda
\mathbb E_{o_0,o_1\sim\pi_\theta}
\left[
    b(o_0,o_1)
    \left(
    \nabla_\theta\log\pi_\theta(o_0\mid x_0)
    +
    \nabla_\theta\log\pi_\theta(o_1\mid x_1)
    \right)
\right].
\end{aligned}
\end{equation}
The first line is the independent correctness gradient. 
The second line is the learnable cross-view consistency gradient. 
It is estimated using the same paired rollouts used to compute the ConsistRoll reward; no additional labels or model components are required.

\paragraph{Step 4: cross-view correction relative to data augmentation.}
Under the canonical-parser condition, Eq.~\eqref{eq:app-prob-objective} gives
\begin{equation}
\label{eq:app-cr-prob-gradient}
    \nabla_\theta\mathcal J_{\mathrm{CR}}(\theta)
    =
    \left(\frac{1}{2}+\lambda p_1(\theta)\right)\nabla_\theta p_0(\theta)
    +
    \left(\frac{1}{2}+\lambda p_0(\theta)\right)\nabla_\theta p_1(\theta).
\end{equation}
Independent data augmentation optimizes the two views separately:
\begin{equation}
\label{eq:app-da-objective-gradient}
    \mathcal J_{\mathrm{DA}}(\theta)
    =
    \frac{1}{2}\big(p_0(\theta)+p_1(\theta)\big),
    \qquad
    \nabla_\theta\mathcal J_{\mathrm{DA}}(\theta)
    =
    \frac{1}{2}\nabla_\theta p_0(\theta)
    +
    \frac{1}{2}\nabla_\theta p_1(\theta).
\end{equation}
Subtracting Eq.~\eqref{eq:app-da-objective-gradient} from Eq.~\eqref{eq:app-cr-prob-gradient} yields
\begin{equation}
\label{eq:app-cross-view-correction}
    \nabla_\theta\mathcal J_{\mathrm{CR}}(\theta)
    -
    \nabla_\theta\mathcal J_{\mathrm{DA}}(\theta)
    =
    \lambda\left(
    p_1(\theta)\nabla_\theta p_0(\theta)
    +
    p_0(\theta)\nabla_\theta p_1(\theta)
    \right).
\end{equation}
This proves Eq.~\eqref{eq:cr-gradient-correction}. 
The correction has a direct learnability interpretation. 
The update for view $0$ is scaled from $\frac{1}{2}\nabla_\theta p_0$ to $(\frac{1}{2}+\lambda p_1)\nabla_\theta p_0$, and symmetrically for view $1$. 
Therefore, once either view has nonzero success probability, its semantically equivalent counterpart increases the gradient magnitude of the paired view. 
The bias is not a post-training averaging operation; it changes the direction and magnitude of the RL update.

\paragraph{Step 5: unbiased Monte Carlo estimator for the unclipped objective.}
For the unclipped population objective, an unbiased single-pair policy-gradient estimator is
\begin{equation}
\label{eq:app-unbiased-estimator}
\begin{aligned}
    g(o_0,o_1)
    &=
    \frac{1}{2}c_0(o_0)\nabla_\theta\log\pi_\theta(o_0\mid x_0)
    +
    \frac{1}{2}c_1(o_1)\nabla_\theta\log\pi_\theta(o_1\mid x_1) \\
    &\quad+
    \lambda b(o_0,o_1)
    \left(
    \nabla_\theta\log\pi_\theta(o_0\mid x_0)
    +
    \nabla_\theta\log\pi_\theta(o_1\mid x_1)
    \right).
\end{aligned}
\end{equation}
Taking expectation over $o_0\sim\pi_\theta(\cdot\mid x_0)$ and $o_1\sim\pi_\theta(\cdot\mid x_1)$ recovers Eq.~\eqref{eq:app-full-gradient-decomposition}. 
Thus,
\begin{equation}
    \mathbb E[g(o_0,o_1)]=\nabla_\theta\mathcal J_{\mathrm{CR}}(\theta).
\end{equation}
In practice, ConsistRoll uses old-policy rollouts, group-relative normalization, and clipping, exactly as in GRPO. 
These ingredients produce the implemented surrogate rather than the idealized population gradient. 
The estimator above is used only to establish that the proposed inductive bias is, in principle, learnable by standard policy-gradient optimization.

\paragraph{Step 6: convergence to first-order stationarity.}
We now prove the non-convex stochastic gradient-ascent bound in Eq.~\eqref{eq:cr-convergence}. 
Let $\mathcal F_t$ denote the randomness up to iteration $t$. 
Assume
\begin{equation}
\label{eq:app-unbiased-variance-assumption}
    \mathbb E[g_t\mid\mathcal F_t]
    =
    \nabla_\theta\mathcal J_{\mathrm{CR}}(\theta_t),
    \qquad
    \mathbb E[\|g_t-\nabla_\theta\mathcal J_{\mathrm{CR}}(\theta_t)\|^2\mid\mathcal F_t]
    \leq
    \sigma^2 .
\end{equation}
Let $\Delta_t=g_t-\nabla_\theta\mathcal J_{\mathrm{CR}}(\theta_t)$. 
The zero-mean property of $\Delta_t$ gives
\begin{equation}
\label{eq:app-second-moment-bound}
\begin{aligned}
    \mathbb E[\|g_t\|^2\mid\mathcal F_t]
    &=
    \mathbb E[\|\nabla_\theta\mathcal J_{\mathrm{CR}}(\theta_t)+\Delta_t\|^2\mid\mathcal F_t] \\
    &=
    \|\nabla_\theta\mathcal J_{\mathrm{CR}}(\theta_t)\|^2
    +
    2\left\langle
    \nabla_\theta\mathcal J_{\mathrm{CR}}(\theta_t),
    \mathbb E[\Delta_t\mid\mathcal F_t]
    \right\rangle
    +
    \mathbb E[\|\Delta_t\|^2\mid\mathcal F_t] \\
    &\leq
    \|\nabla_\theta\mathcal J_{\mathrm{CR}}(\theta_t)\|^2+
    \sigma^2 .
\end{aligned}
\end{equation}
Since $\mathcal J_{\mathrm{CR}}$ is $L$-smooth, for the ascent update $\theta_{t+1}=\theta_t+\eta g_t$,
\begin{equation}
\label{eq:app-smoothness-step}
    \mathcal J_{\mathrm{CR}}(\theta_{t+1})
    \geq
    \mathcal J_{\mathrm{CR}}(\theta_t)
    +
    \eta\langle \nabla_\theta\mathcal J_{\mathrm{CR}}(\theta_t),g_t\rangle
    -
    \frac{L\eta^2}{2}\|g_t\|^2 .
\end{equation}
Taking conditional expectation and using Eq.~\eqref{eq:app-second-moment-bound}, we obtain
\begin{equation}
\label{eq:app-one-step-progress-detailed}
\begin{aligned}
    \mathbb E[\mathcal J_{\mathrm{CR}}(\theta_{t+1})\mid\mathcal F_t]
    &\geq
    \mathcal J_{\mathrm{CR}}(\theta_t)
    +
    \eta\|\nabla_\theta\mathcal J_{\mathrm{CR}}(\theta_t)\|^2
    -
    \frac{L\eta^2}{2}
    \left(
    \|\nabla_\theta\mathcal J_{\mathrm{CR}}(\theta_t)\|^2+
    \sigma^2
    \right) \\
    &=
    \mathcal J_{\mathrm{CR}}(\theta_t)
    +
    \eta\left(1-\frac{L\eta}{2}\right)
    \|\nabla_\theta\mathcal J_{\mathrm{CR}}(\theta_t)\|^2
    -
    \frac{L\eta^2\sigma^2}{2} .
\end{aligned}
\end{equation}
When $\eta\leq 1/L$, we have $1-L\eta/2\geq 1/2$, and hence
\begin{equation}
\label{eq:app-one-step-progress-simple}
    \mathbb E[\mathcal J_{\mathrm{CR}}(\theta_{t+1})\mid\mathcal F_t]
    \geq
    \mathcal J_{\mathrm{CR}}(\theta_t)
    +
    \frac{\eta}{2}\|\nabla_\theta\mathcal J_{\mathrm{CR}}(\theta_t)\|^2
    -
    \frac{L\eta^2\sigma^2}{2} .
\end{equation}
Taking total expectation, rearranging, and summing over $t=0,\ldots,T-1$ gives
\begin{equation}
\label{eq:app-sum-progress}
\begin{aligned}
    \frac{\eta}{2}
    \sum_{t=0}^{T-1}
    \mathbb E[\|\nabla_\theta\mathcal J_{\mathrm{CR}}(\theta_t)\|^2]
    &\leq
    \sum_{t=0}^{T-1}
    \mathbb E[\mathcal J_{\mathrm{CR}}(\theta_{t+1})-\mathcal J_{\mathrm{CR}}(\theta_t)]
    +
    \frac{T L\eta^2\sigma^2}{2} \\
    &=
    \mathbb E[\mathcal J_{\mathrm{CR}}(\theta_T)]-\mathcal J_{\mathrm{CR}}(\theta_0)
    +
    \frac{T L\eta^2\sigma^2}{2} \\
    &\leq
    \mathcal J_{\mathrm{CR}}^*-\mathcal J_{\mathrm{CR}}(\theta_0)
    +
    \frac{T L\eta^2\sigma^2}{2} .
\end{aligned}
\end{equation}
Dividing both sides by $\eta T/2$ yields
\begin{equation}
    \frac{1}{T}\sum_{t=0}^{T-1}
    \mathbb E[\|\nabla_\theta\mathcal J_{\mathrm{CR}}(\theta_t)\|^2]
    \leq
    \frac{2(\mathcal J_{\mathrm{CR}}^*-\mathcal J_{\mathrm{CR}}(\theta_0))}{\eta T}
    +
    L\eta\sigma^2,
\end{equation}
which proves Eq.~\eqref{eq:cr-convergence}. 
With $\eta=\mathcal O(T^{-1/2})$, both terms on the right-hand side scale as $\mathcal O(T^{-1/2})$. 
This proves convergence to a first-order stationary point in the standard average-gradient-norm sense.

\subsection{Compatibility of Cross-View Consistency}
\label{app:compatibility-consistency}

Proposition~\ref{prop:view-coupled-policy-improvement} is the main optimization-level statement. 
For completeness, we also record a weaker distributional compatibility result: enforcing identical answer distributions across semantically-invariant views does not reduce view-set-averaged correctness at the distribution level. 
This lemma supports the validity intuition, but it is not the main theoretical claim.

Let the answer distribution induced by a policy be
\begin{equation}
    p_\pi(a\mid I,q)=\sum_{o\in\mathcal O}\mathbbm{1}[\hat a(o)=a]\pi(o\mid I,q).
\end{equation}
A policy is cross-view consistent on $\mathcal V_i$ if
\begin{equation}
\label{eq:app-answer-dist-consistency}
    p_\pi(a\mid I_i^{(m)},q_i)
    =
    p_\pi(a\mid I_i^{(\ell)},q_i),
    \qquad
    \forall a,\; m,\ell\in[K].
\end{equation}

\begin{lemma}[Non-destructive consistency over semantically-invariant views]
\label{lem:consistent-optimum}
Under Assumption~\ref{asm:semantic-invariance}, for any policy $\pi$, there exists a cross-view consistent policy distribution $\bar\pi_i$ over $\mathcal V_i$ whose average correctness over $\mathcal V_i$ is exactly the same as that of $\pi$.
\end{lemma}

\begin{proof}
For a fixed sample $i$, define the view-averaged distribution
\begin{equation}
\label{eq:app-view-avg-policy}
    \bar\pi_i(o\mid I_i^{(m)},q_i)
    =
    \frac{1}{K}\sum_{\ell=1}^{K}\pi(o\mid I_i^{(\ell)},q_i),
    \qquad
    \forall m\in[K].
\end{equation}
First, $\bar\pi_i$ is a valid distribution because it is an average of valid distributions:
\begin{equation}
    \sum_{o\in\mathcal O}\bar\pi_i(o\mid I_i^{(m)},q_i)
    =
    \frac{1}{K}\sum_{\ell=1}^{K}\sum_{o\in\mathcal O}\pi(o\mid I_i^{(\ell)},q_i)
    =1,
    \qquad
    \bar\pi_i(o\mid I_i^{(m)},q_i)\geq 0.
\end{equation}
Second, the right-hand side of Eq.~\eqref{eq:app-view-avg-policy} is independent of $m$, so $\bar\pi_i$ induces the same answer distribution for every view and is cross-view consistent.
Let $c_i(o)=\mathbbm{1}[\hat a(o)=y_i^*]$. 
By Assumption~\ref{asm:semantic-invariance}, the same correctness function applies to every view in $\mathcal V_i$. 
The averaged correctness of $\bar\pi_i$ is
\begin{equation}
\label{eq:app-acc-preservation}
\begin{aligned}
    \frac{1}{K}\sum_{m=1}^{K}
    \mathbb E_{o\sim \bar\pi_i(\cdot\mid I_i^{(m)},q_i)}[c_i(o)]
    &=
    \frac{1}{K}\sum_{m=1}^{K}
    \sum_{o\in\mathcal O} c_i(o)
    \left(
    \frac{1}{K}\sum_{\ell=1}^{K}\pi(o\mid I_i^{(\ell)},q_i)
    \right) \\
    &=
    \sum_{o\in\mathcal O} c_i(o)
    \left(
    \frac{1}{K}\sum_{\ell=1}^{K}\pi(o\mid I_i^{(\ell)},q_i)
    \right) \\
    &=
    \frac{1}{K}\sum_{\ell=1}^{K}
    \sum_{o\in\mathcal O}c_i(o)\pi(o\mid I_i^{(\ell)},q_i) \\
    &=
    \frac{1}{K}\sum_{\ell=1}^{K}
    \mathbb E_{o\sim \pi(\cdot\mid I_i^{(\ell)},q_i)}[c_i(o)] .
\end{aligned}
\end{equation}
Thus, $\bar\pi_i$ is cross-view consistent and exactly preserves transformation-averaged correctness.
\end{proof}

\subsection{Why Consistency Alone Is Insufficient}
\label{app:consistency-only}

The validity result relies on a correctness-gated consistency bonus. 
A consistency-only reward can be maximized by policies that are stable but wrong.

Consider two semantically invariant views $(I_i,q_i)$ and $(T(I_i),q_i)$. 
Let $o$ and $o'$ be the responses sampled from these two views. 
A naive consistency reward can be written as
\begin{equation}
\label{eq:consistency-only-reward}
    r_{\mathrm{con}}(o,o')
    =
    \mathbbm{1}[\hat a(o)=\hat a(o')] .
\end{equation}
For any fixed answer $a_0$, define a constant-answer policy $\pi_{a_0}$ that always emits a response whose parsed answer is $a_0$, independent of the image and question. 
Then every paired rollout satisfies $\hat a(o)=\hat a(o')=a_0$, so
\begin{equation}
    \mathbb E_{o,o'\sim\pi_{a_0}}[r_{\mathrm{con}}(o,o')]=1.
\end{equation}
However, if $a_0\neq y_i^*$, the same policy has zero task correctness on sample $i$:
\begin{equation}
\label{eq:consistency-degenerate}
    \mathbb E_{o\sim\pi_{a_0}}
    \left[
    \mathbbm{1}[\hat a(o)=y_i^*]
    \right]
    =0,
    \qquad
    a_0\neq y_i^* .
\end{equation}
Thus, consistency alone admits a degenerate optimum: the policy can ignore the image and question while producing the same incorrect answer across views.

ConsistRoll avoids this degeneration by making the consistency signal correctness-gated. 
For paired rollouts $(o,o')$, we use
\begin{equation}
\label{eq:correctness-gated-bonus}
    b(o,o')
    =
    \mathbbm{1}
    \left[
    \hat a(o)=y_i^*
    \;\land\;
    \hat a(o')=y_i^*
    \;\land\;
    \kappa(o,o')=1
    \right].
\end{equation}
The first two conjuncts rule out wrong-but-consistent answers, while the last conjunct explicitly favors stability across semantically invariant views. 
When the verifier maps all correct responses to a unique canonical answer, the consistency check may be logically implied by dual correctness. 
In practice, however, parsers often tolerate multiple surface forms or partially normalized expressions, so the explicit consistency check remains useful.

\subsection{How ConsistRoll Changes the GRPO Advantage}
\label{app:advantage}

We next connect the population analysis to the finite-sample GRPO advantage used by the algorithm. 
For clarity, consider one sample with two semantically invariant views: the original image $I_i$ and a transformed image $T(I_i)$. 
Let $n$ denote the total rollout budget of the coupled group and let $M=n/2$ denote the number of rollouts sampled from each view. 
Let $v\in\{0,1\}$ index the view and $k\in[M]$ index the rollout sampled from that view. 
The old policy generates responses $o_{v,k}$.

In standard augmented GRPO, each view is treated as an independent training instance. 
With the view-local correctness reward 
$c_{v,k}=\mathbbm{1}[\hat a(o_{v,k})=y_i^*]$, the advantage is normalized within each view:
\begin{equation}
\label{eq:independent-da-advantage}
    A^{\mathrm{DA}}_{v,k}
    =
    \frac{c_{v,k}-\bar c_v}{s_v+\delta},
    \qquad
    \bar c_v=\frac{1}{M}\sum_{k=1}^{M}c_{v,k},
    \qquad
    s_v^2=\frac{1}{M}\sum_{k=1}^{M}(c_{v,k}-\bar c_v)^2 .
\end{equation}
This objective increases input diversity but compares a rollout only against other rollouts from the same view.

ConsistRoll instead places both views in the same rollout group and assigns a correctness-gated pair bonus. 
For each paired rollout index $k$, define
\begin{equation}
\label{eq:consistroll-reward-def}
\begin{gathered}
    b_k
    =
    \mathbbm{1}
    \left[
    c_{0,k}=1
    \;\land\;
    c_{1,k}=1
    \;\land\;
    \kappa(o_{0,k},o_{1,k})=1
    \right],
    \qquad
    \tilde r_{v,k}=c_{v,k}+\lambda b_k,                                                        \\
    \tilde\mu
    =
    \frac{1}{n}\sum_{v=0}^{1}\sum_{k=1}^{M}\tilde r_{v,k},
    \qquad
    \tilde\sigma^2
    =
    \frac{1}{n}\sum_{v=0}^{1}\sum_{k=1}^{M}
    (\tilde r_{v,k}-\tilde\mu)^2,
    \qquad
    \tilde A_{v,k}
    =
    \frac{\tilde r_{v,k}-\tilde\mu}{\tilde\sigma+\delta} .
\end{gathered}
\end{equation}
Here $\lambda\geq 0$ controls the strength of the consistency bonus.

To make the coupling explicit, let
\begin{equation}
    \bar c=\frac{1}{2}(\bar c_0+\bar c_1),
    \qquad
    \bar b=\frac{1}{M}\sum_{k=1}^{M}b_k .
\end{equation}
Since the same bonus $b_k$ is assigned to both views, we have
\begin{equation}
    \tilde\mu
    =
    \frac{1}{n}\sum_{v=0}^{1}\sum_{k=1}^{M}(c_{v,k}+\lambda b_k)
    =
    \frac{1}{2}(\bar c_0+\bar c_1)+\lambda\bar b
    =
    \bar c+\lambda\bar b .
\end{equation}
Therefore, the numerator of the ConsistRoll advantage admits the exact decomposition
\begin{equation}
\label{eq:consistroll-advantage-decomposition}
\begin{aligned}
    \tilde r_{v,k}-\tilde\mu
    &=
    c_{v,k}+\lambda b_k-\bar c-\lambda\bar b \\
    &=
    \underbrace{(c_{v,k}-\bar c_v)}_{\text{within-view correctness}}
    +
    \underbrace{(\bar c_v-\bar c)}_{\text{shared cross-view baseline}}
    +
    \underbrace{\lambda(b_k-\bar b)}_{\text{pair-level consistency correction}}, \\
    \tilde A_{v,k}
    &=
    \frac{
    (c_{v,k}-\bar c_v)+(\bar c_v-\bar c)+\lambda(b_k-\bar b)
    }{\tilde\sigma+\delta} .
\end{aligned}
\end{equation}
Eq.~\eqref{eq:consistroll-advantage-decomposition} reveals the finite-sample analogue of Proposition~\ref{prop:view-coupled-policy-improvement}. 
The first term is the usual within-view correctness signal. 
The second term compares each view against a shared cross-view baseline, so a transformed-view failure can be contrasted with original-view success in the same group. 
The third term is the explicit consistency correction: paired rollouts receive extra relative advantage only when both views are correct and answer-consistent. 
Thus, ConsistRoll makes the inductive bias learnable in the actual GRPO update by changing credit assignment from pointwise correctness to view-coupled correctness.

\paragraph{Connection to augmented GRPO.}
When $\lambda=0$, the explicit consistency bonus disappears. 
If the two views are still normalized in the same group, the method becomes pooled augmented GRPO: it keeps the shared cross-view baseline term $(\bar c_v-\bar c)$ but removes the pair-level consistency correction. 
If each view is normalized independently, then $\lambda=0$ reduces to standard GRPO with data augmentation. 
Therefore, the empirical gap between $\lambda=0$ and $\lambda>0$ isolates the effect of the explicit consistency-driven credit assignment.

\end{document}